\documentclass[11pt]{article} %, twocolumn]{article}

\usepackage[dvips]{graphics}
\DeclareGraphicsExtensions{.eps.gz,.eps,.epsi.gz,.epsi,.ps,.ps.gz}
\usepackage{epsfig}
\usepackage{color}
\graphicspath{{figures/}} \textwidth=6.5in \textheight=9in

\usepackage{latexsym}
\usepackage{graphicx}
\usepackage{amsfonts}
\usepackage{amssymb}
\usepackage{mathrsfs}
\usepackage{verbatim}

\usepackage{amsmath}
\usepackage{latexsym}
\usepackage{mathrsfs} % for math script
\usepackage{amsfonts}
\usepackage{amssymb}

\usepackage{graphicx}
\usepackage{verbatim}
\usepackage{geometry}

\newcommand{\bel}{\begin{eqnarray}\label}
\newcommand{\eel}{\end{eqnarray}}
\newcommand{\bes}{\begin{eqnarray*}}
\newcommand{\ees}{\end{eqnarray*}}

\def\benu{\begin{enumerate}}
\def\eenu{\end{enumerate}}

\def\argmin{\mathop{\rm arg\, min}}

\def\real{{\mathbb{R}}}

\def\complex{\mathop{{\rm I}\kern-.58em\hbox{\rm C}}\nolimits}
\def\pa{\partial}

\def\rank{\hbox{rank}}

\def\sgn{\hbox{sgn}}\def\sgn{\hbox{\rm sgn}}

\def\Var{\hbox{Var}}

\def\supp{\hbox{supp}}

\def\mathbold{\boldsymbol} %\def\mathbold{\mathbf}
\def\ba{\mathbold{a}}

\def\bb{\mathbold{b}}

\def\bD{\mathbold{D}}

\def\bI{\mathbold{I}}

\def\mtil{\widetilde{m}}

\def\scrM{{\mathscr M}}

\def\bN{\mathbold{N}}

\def\bP{\mathbold{P}}

\def\Shat{\widehat{S}}\def\Stil{{\widetilde S}}

\def\bu{\mathbold{u}}

\def\bU{\mathbold{U}}

\def\bv{\mathbold{v}}

\def\bV{\mathbold{V}}

\def\bW{\mathbold{W}}

\def\bx{\mathbold{x}}

\def\bX{\mathbold{X}}
\def\scrX{{\mathscr X}}

\def\by{\mathbold{y}}

\def\bbeta{\mathbold{\beta}}\def\hbeta{\widehat{\beta}}
\def\tbeta{\widetilde{\beta}}
\def\hbbeta{{\widehat{\bbeta}}}\def\tbbeta{{\widetilde{\bbeta}}}

\def\bDelta{\mathbold{\Delta}}

\def\ep{\varepsilon}\def\eps{\epsilon}
\def\bep{\mathbold{\ep}}

\def\lam{\lambda}

\def\tnu{\widetilde{\nu}}

\def\hsigma{\widehat{\sigma}}

\def\bSigma{\mathbold{\Sigma}}

\usepackage{amsmath,amsthm,amssymb}
\usepackage{latexsym}
\usepackage{mathrsfs}
\usepackage{amsfonts}

\usepackage{fullpage}
\usepackage{color}

\newtheorem{theorem}{Theorem}
\newtheorem{lemma}{Lemma}
\newtheorem{proposition}{Proposition}
\newtheorem{corollary}{Corollary}
\newtheorem{remark}{Remark}
\newtheorem{definition}{Definition}
\newtheorem{assumption}{Assumption}

\def\argmin{\mathop{\rm arg\, min}}
\def\real{\mathop{{\rm I}\kern-.2em\hbox{\rm R}}\nolimits}

\newcommand{\R}{{\mathbb{R}}}
\newcommand{\obeta}{{\hbbeta^{o}}}
\newcommand{\drho}{{\dot\rho}}

\def\RE{\hbox{\rm RE}}
\def\CIF{\hbox{\rm CIF}}
\def\RIF{\hbox{\rm RIF}}

\title{A General Theory of Concave Regularization for High Dimensional Sparse Estimation Problems}
\author{Cun-Hui Zhang\thanks{Research partially supported by the
NSF Grants DMS 0906420, DMS-11-06753
and NSA Grant H98230-11-1-0205} \\
Department of Statistics\\ and Biostatistics\\
Rutgers University, NJ \\
{\it czhang@stat.rutgers.edu} \\
\and
Tong Zhang\thanks{Research partially supported by the following grants: AFOSR-10097389, NSA
-AMS 081024, NSF DMS-1007527, and NSF IIS-1016061} \\
Department of Statistics\\ and Biostatistics\\
Rutgers University, NJ\\
{\it tzhang@stat.rutgers.edu}}

\begin{document}
\date{}
\maketitle{}

\begin{abstract}
  Concave regularization methods provide natural procedures for sparse recovery. 
  However, they are difficult to analyze in the high dimensional setting. Only recently a  
  few sparse recovery results have been established for some specific local solutions obtained 
  via specialized numerical procedures. 
  Still, the fundamental relationship between these solutions such as whether they are 
  identical or their relationship to the global minimizer of the underlying nonconvex 
  formulation is unknown. 
  The current paper fills this conceptual gap by presenting a general theoretical framework showing that under appropriate conditions, the
  global solution of nonconvex regularization leads to desirable recovery performance; moreover, 
  under suitable conditions, the global solution corresponds to the unique sparse local solution, which can
  be obtained via different numerical procedures. 
  Under this unified framework, we present an overview of existing results and discuss their connections.
  The unified view of this  work leads to a more satisfactory 
  treatment of concave high dimensional sparse estimation procedures, and serves as guideline for developing
 further numerical procedures for concave regularization.
\end{abstract}

\section{Introduction}

Let $\bX$ be an $n \times p$ design matrix and $\by \in \R^n$ a response vector satisfying 
\begin{equation}
\by = \bX \bbeta  + \bep ,
\label{eq:noise-def}
\end{equation}
where $\bbeta \in \R^p$ is a target vector of regression coefficients and 
$\bep \in \R^n$ is a noise vector.  
This paper concerns the estimation of the value of $\bX\bbeta$, that of $\bbeta$, or its support set 
$\supp(\bbeta)$, where $\supp(\bb) := \{j: b_j \neq 0\}$ for any vector 
$\bb = (b_1,\ldots,b_p)^\top \in \R^p$.

We are interested in the high-dimensional case where $n$ and $p$ are both allowed to 
diverge, including the case of $p\gg n$. We assume that the target vector $\bbeta$ is sparse
in some sense; such as the $\ell_0$ sparsity $|\supp(\bbeta)|\le c_0n/\ln p$, or the capped-$\ell_1$ sparsity 
$\sum_{j=1}^p\min(1,|\beta_j/\sigma|\sqrt{n/\ln p}) \le c_0 n/\ln p$,  
where $\sigma$ is a certain noise level and $c_0$ is a fixed small constant. 
While we are mainly interested in the Gaussian noise $\bep\sim N(0, \sigma^2\bI_{n\times n})$ 
or zero-mean sub-Gaussian noise, the specific noise properties required in our analysis will be provided later.

We consider the following class of penalized least squares estimators 
\begin{equation}
\hbbeta := \argmin_{\bb \in \R^p}L_\lam(\bb), \qquad 
L_\lam(\bb) := \frac{1}{2n}\|\by -\bX\bb\|_2^2 + \sum_{j=1}^p \rho(b_j;\lam), 
\label{eq:hbbeta}
\end{equation}
where $\bb = (b_1,\ldots,b_p)^\top$ and $\rho(t;\lam)$
is a scalar regularization function with a certain regularization parameter $\lam>0$. 
%We focus on the situation that each scalar function $\rho(t;\lam)$ is concave and non-decreasing in $t \geq 0$, although
%some results are obtained under a weaker condition of subadditivity.
As an example, we may let $\rho(t; \lambda)=\lambda^2 I(t \neq 0)/2$, which corresponds to the
$\ell_0$ regularization problem. Here $I(\cdot)$ denotes $\{0,1\}$ valued indicator function.
Since $I(t \neq 0)$ is a discontinuous function at $t=0$, the corresponding $\ell_0$
optimization problem may be difficult to solve. In practice, one also looks at continuous
regularizers that approximate $\ell_0$ regularization, such as $\rho(t;\lam) = \min(\lam^2/2,\lam |t|)$. 
As we will show in the paper, sparse local solutions of such regularizers can be obtained using 
standard numerical procedures 
(such as gradient descent), and they are closely related to the global solution of (\ref{eq:hbbeta}).

\section{Survey of Existing Concave Regularization Results}
\label{sec:survey}
While this survey is not intended to be comprehensive, it presents a high-level view
of some important contributions to the area of concave regularization. 
We will discuss both methodological and analytical contributions. 

\subsection{Terminologies}
\label{sec:terminology}
The following notation is used throughout the paper. 
For any dimension $d$, bold face letters denote vectors and normal face their elements, 
e.g. $\bv=(v_1,\ldots,v_d)^\top$, with $\supp(\bv)$ being its support $\{j:v_j\neq 0\}\cap\{0,\ldots,d\}$. 
Capital bold face letters denote 
matrices, e.g. $\bX$ and $\bSigma$. The $\ell_q$ ``norm'' of $\bv$ is 
$\|\bv\|_q := \big(\sum_{j=1}^d |v_j|^q\big)^{1/q}$ for $0<q<\infty$, with the usual extension 
$\|\bv\|_0 := |\supp(\bv)|$ and $\|\bv\|_\infty := \max_{j\le d}|v_j|$.  
Design vectors, or columns of $\bX$, are denoted by $\bx_j$. 
For simplicity, we assume throughout the paper that 
the columns $\bX$ are normalized to 
\[
\|\bx_j\|_2=\sqrt{n} .
\]
This condition is not essential but it simplifies some notations.
For variable sets $A\subseteq\{1,\ldots,p\}$, 
$\bX_A=(\bx_j, j\in A)$ denotes the restriction of columns of $\bX$ to $A$,
and $\bb_A=(b_j, j\in A)^\top$ the restriction of vector $\bb\in\R^p$ to $A$. 
The maximum and minimum eigenvalues of matrix $\bSigma$ are denoted by 
$\lambda_{\max}(\bSigma)$ and $\lambda_{\min}(\bSigma)$. 
%The support set of the target vector $\bbeta$ in (\ref{eq:noise-def}) is denoted by $S=\supp(\bbeta)$. 

%We introduce the following terminologies so that results can be more easily compared.

\begin{definition} \label{def:terminology}
The following terminologies will be used to simplify discussion. 
\begin{itemize}
\item [(a)] The $\ell_0$ sparsity of $\bbeta$ means $\|\bbeta\|_0\le s^*$. 
To allow $\bbeta$ with many more components near zero, a weaker notion of 
capped-$\ell_1$ sparsity is $\sum_j\min(1,|\beta_j|/\lam_{univ})\le s^*$, 
where $\lam_{univ}=\sigma\sqrt{(2/n)\ln p}$ is the universal threshold level 
for a certain noise level $\sigma$. 
\item [(b)] A regularity condition on $\bX$ is a class $\scrX$ of (column-normalized) matrices that match a 
sparsity condition on $\bbeta$ to guarantee a desired result. Such a regularity condition can be stated as 
$\bX\in \scrX_{s^*}^{n\times p}$, with matrix classes 
$\scrX_{s^*}^{n\times p}\subseteq\R^{n\times p}$ indexed by $(n,p,s^*)$, 
where $s^*$ is the sparsity level of the matching regularity condition on $\bbeta$.  
Such a condition on $\bX$ is called an $\ell_2$ regularity condition 
(or simply $\ell_2$ regular) if the matrix classes $\scrX_{s^*}^{n\times p}$ are sufficiently large to satisfy 
the following condition:
\begin{itemize}
%\item For $p\le n$ and all sparsity levels $s^*$, 
%$\scrX_{s^*}^{n\times p}$ contains all $n \times p$ matrices of rank $p$. 
%\item There exists a certain constant $u^*>0$ such that for all $(n,p,s^*)$, $\scrX_{s^*}^{n\times p}$ 
%contains all $n \times p$ matrices 
%$\bX$ satisfying $\max_{j\le p}\|\bx_j\|^2/\lam_{\min}(\bX^\top \bX)\le u^*$. 
\item Given any $u_0\ge 1$, there exists a constant $c_0>0$ such that for all $0<\delta \le 1/e$ 
\bes
\inf_{\mu,n,p,s^*}\Big\{ \mu(Q^{-1}(\scrX^{n\times p}_{s^*})): \mu\in\scrM^{n\times p}_{u_0}, 
(s^*/n)\ln(p/\delta) \le c_0, \min(n,p,s^*)\ge 1\Big\} \ge 1-\delta, 
\ees 
where $\scrM^{n\times p}_{u_0}$ is the set of probability measures in $\R^{n\times p}$ 
under which the rows of $\R^{n\times p}$ are iid $N(0,\bSigma)$ for some $\bSigma$ 
with $\lambda_{\max}(\bSigma)/\lambda_{\min}(\bSigma) \leq u_0$ and 
identical diagonal elements, and 
$Q$ is the column normalization mapping given by 
$Q(\bX) = (\bx_jn^{1/2}/\|\bx_j\|_2, j\le p)$. 
\end{itemize}
\item [(c)] An estimator $\hbbeta$ is selection consistent if $\supp(\hbbeta)=\supp(\bbeta)$, and 
sign-consistent if $\sgn(\hbbeta)=\sgn(\bbeta)$, with the convention $\sgn(0)=0$ for the sign function. 
%Given a set regularity conditions in the regression model (\ref{eq:noise-def}),  
%an estimator $\hbbeta$ is selection consistent if 
%$P\{\supp(\hbbeta)=\supp(\bbeta)\}\to 1$, and sign-consistent if $P\{\sgn(\hbbeta)=\sgn(\bbeta)\}\to 1$, 
%with the convention $\sgn(0)=0$ for the sign function. Here the limit is taken as $n\to\infty$ 
%uniformly in $p$ allowed by the regularity conditions. 
%Regularity conditions in theorems in this paper always allow a range of $p$ up to 
%$p \le e^{a_0n}$ for a small $a_0>0$. 
\item [(d)] An estimator has the oracle property if
\bel{oracle-property}
\hbbeta= \obeta , \quad \obeta_S=(\bX_S^\top\bX_S)^{-1}\bX_S^\top\by , \;
\supp(\obeta) \subseteq S , 
\eel 
where $S=\supp(\bbeta)$. The estimator $\obeta$ is called the oracle LSE.
\end{itemize}
\end{definition}

\begin{remark}
The standard regularity condition for the classical low-dimensional statistical scenario of 
$p \le n$ is that the rank of $\bX$ is $p$. 
Definition~\ref{def:terminology} (b) generalizes this classical regularity condition to allow $p \gg n$. 
We may explicitly include the classical situation into the definition of $\ell_2$ regularity
(that is, require $\scrX_{s^*}^{n \times p}$ to contain all column-normalized $n \times p$ matrices 
of rank $p$) if we confine our discussion to fixed sample conditions. 
See %Remark \ref{remark-CIF-sparse-eigen} and 
the last paragraph of this subsection for more discussion. 
%with $p$-th singular value bounded away from zero); 
%and in such case all of our results still hold with minor modifications of theorem statements 
%to accommodate the scenario.
\end{remark}

\begin{remark}
  If we consider a sequence of models in (\ref{eq:noise-def}) with $n \to \infty$, then asymptotically
  an estimator has the oracle property (allowing statistical inference for all linear functionals of $\bbeta$) if 
  \[
  \hbox{$\sup_{\ba}$}\ P\big\{|\ba^\top(\hbbeta-\obeta)|^2 > \eps \Var(\ba^T\obeta)\big\}
  =o(1)\ \forall\eps>0,
  \]
  and this is a weaker requirement than (\ref{oracle-property}) because it allows
  $\hbbeta$ to converge only asymptotically to $\obeta$.
  While this work focuses on the stronger requirement (\ref{oracle-property}) that is easier to
  interpret in the finite sample situation, the weaker definition has been
  used in some previous asymptotic analysis.
\end{remark}

For $0<r\le 1$, the capped-$\ell_1$ sparsity condition holds for all vectors with 
$\|\bbeta\|_r\le R$ as long as $(R/\lam_{univ})^r\le s^*$. 

In the classical statistical scenario of $p <n$, a standard regularity condition on the 
design matrix $\bX$ is that the rank of $\bX$ is $p$.
Definition~\ref{def:terminology} (b) generalizes this classical regularity condition to $p \gg n$. 
For example, $\inf_{|A|\le 3s^*}\{\rank(\bX_A)/|A|\} = 1$ is $\ell_2$ regular. 
The $\ell_2$ notion allows an assessment of the strength of assumptions 
on $\bX$ by random matrix theory without repeating technical statements of more specialized 
conditions. Moreover, since the $\ell_2$ criterion is required to hold for 
$(s^*/n)\ln(p/\delta)\le c_0$, results based on $\ell_2$ regularity condition on $\bX$ 
and matching sparsity condition of $\bbeta$ must apply to the case of large $p$, 
including $p\gg n$. Since regularity conditions on $\bbeta$ and $\bX$ must work together 
to guarantee their consequences, for simplicity the sparsity level $s^*$ for $\ell_2$ regularity 
is always understood in the sequel as the $\ell_0$ or capped-$\ell_1$ sparsity level of $\bbeta$ 
given in Definition \ref{def:terminology} (a). 

Throughout the paper, $\bX$ and $\bbeta$ in (\ref{eq:noise-def}) are treated as deterministic. 
Since the $\ell_2$ criterion is about the size of $\scrX^{n\times p}_{s^*}$, it does not imply 
randomness of $\bX$. In fact, since the $\ell_2$ criterion is required to hold simultaneously for 
all $\mu\in \scrM^{n\times p}_{u_0}$ with the same $\scrX^{n\times p}_{s^*}$ in $\R^{n\times p}$, 
an $\ell_2$ regularity condition is weaker than the condition of a random $\bX$ 
with distribution $\mu(Q^{-1}(\cdot))$ for a fixed $\mu\in \scrM^{n\times p}_{u_0}$ 
and typically requires a more explicit specification of the matrix class 
$\scrX^{n\times p}_{s^*}$. We call the criterion $\ell_2$, since it depends only on the 
range of the spectrum (the smallest and largest eigenvalues) of $\bSigma$. 

%The $\ell_2$-regularity condition is defined on all deterministic design
%matrices $\bX$; it further requires that with large probability the condition holds for any 
%$\bX$ that is constructed by drawing $n$ random $p$ dimensional vectors from $N(0,\bSigma)$.

The rest of the subsection discusses different forms of $\ell_2$ conditions. 
Since the meaning of sparsity level is always clear in its proper context, 
for simplicity we will discuss design matrix conditions without 
explicitly referring to their sparsity levels.
%Different forms of $\ell_2$ conditions will be discussed at the end of this subsection.

%Since $\supp(\ba)\cap\supp(\bbeta)=\emptyset$ is allowed in (\ref{oracle-property}), 
%The oracle property implies selection consistency. On the other hand, (\ref{oracle-property}) holds 
%for the LSE $(\bX_{\Shat}^\top\bX_{\Shat})^{-1}\bX_{\Shat}^\top\by$ with $\Shat=\supp(\hbbeta)$ 
%when $\hbbeta$ is selection consistent. 

In what follows, we will briefly explain some $\ell_2$-regularity conditions appeared in the literature. 
Related conditions have been introduced first in the compressive sensing literature to analyze 
$\ell_1$-regularized recovery of a sparse $\bbeta$ from its random projection $\bX\bbeta$ 
with iid $N(0,1)$ entries in $\bX$. 
The most well-known of such conditions is the restricted isometry condition (RIP) introduced in \cite{CandesTao05}. 
In order to explain RIP,
we first define the lower and upper sparse eigenvalues as
\bel{sparse-eigenvalues}
\kappa_-(m) := \min_{\|\bu\|_0 \leq m; \|\bu\|_2=1} \|\bX \bu\|_2^2/n ,
\quad
\kappa_+(m) := \max_{\|\bu\|_0 \leq m,\|\bu\|_2=1}\|\bX \bu\|_2^2/n . 
\eel
RIP requires $\delta_k+\delta_{2k}+\delta_{3k}< 1$ 
with $k=\|\bbeta\|_0$ and $\delta_m=\max\{\kappa_+(m)-1,1-\kappa_-(m)\}$. 
A related condition is  the uniform uncertainty principle (UUP) $\delta_{2k}+\theta_{2k,k}<1$ in \cite{CandesTao07},
where $\theta_{k,\ell}=\max(\bX_A\bv_A)^\top(\bX_B\bu_B)/n$ with $A\cap B=\emptyset$, $|A|=k$, $|B|=\ell$, and $\|\bu\|_2=\|\bv\|_2=1$. 
For $\ell_1$ regularized estimators, bounds of the optimal order for the $\ell_2$-norm estimation error 
$\|\hbbeta-\bbeta\|_2$ can be obtained under RIP, UUP, as well as their improvement 
$\delta_{1.25k}+\theta_{1.25k,k}<1$ in \cite{CaiWX10}.
While the conditions for RIP and UUP are specialized to hold for random designs with covariance 
matrix $\bSigma = \bI_{p\times p}$, 
related conditions using sparse eigenvalues can be defined to fulfill the $\ell_2$ criterion in 
Definition \ref{def:terminology}(b); for example the sparse Riesz condition (SRC) 
$\|\bbeta\|_0 < \max_m 2m/\{1+\kappa_+(m)/\kappa_-(m)\}$ 
%$k \le \max_m m(2-2\alpha)/\{1-2\alpha+\kappa_+(m)/\kappa_-(m)\}$ with an $\alpha\in (0,1)$ 
in \cite{ZhangHuang08, Zhang10-mc+}, 
and some other extensions in \cite{ZhangT09,YeZ10}. 
These more general conditions are $\ell_2$ regularity conditions by our definition, and they lead to $\ell_2$-norm estimation error bounds
of the optimal order for $\ell_1$ regularized estimators. 
Additional refinements were introduced in the literature, such as the restricted eigenvalue 
of \cite{BickelRT09,Koltchinskii09}, 
\[
\RE_2=\RE_2(\xi,S):=\inf_{\bu}\left\{ \|\bX\bu\|_2/(\|\bu\|_2 n^{1/2}) : \|\bu_{S^c}\|_1 < \xi\|\bu_S\|_1 \right\}
\]
where $S=\supp(\bbeta)$, and the compatibility factor of \cite{vandeGeer07,vandeGeerB09}, 
\[
\RE_1=\RE_1(\xi,S):=\inf_{\bu} \left\{|S|^{1/2}\|\bX\bu\|_2/(\|\bu_S\|_1 n^{1/2}):\|\bu_{S^c}\|_1 < \xi\|\bu_S\|_1 \right\}. 
\]
It can be shown that $\RE_1\geq \RE_2$ and appropriate sparse eigenvalues imply $\RE_2>0$. 
Therefore both $\RE_2$ and $\RE_1$ are $\ell_2$ regularity conditions.
Moreover, for $\ell_1$ regularized estimators, $\RE_1$ provides 
$\ell_1$-norm estimation and $\ell_2$-norm prediction error bounds of optimal order, 
and $\RE_2$ provides $\ell_2$-norm estimation bounds of optimal order. 

This paper employs an even weaker condition involving 
a restricted invertibility factor $\RIF_q$ in (\ref{eq:RIF}) which is
related to the cone invertibility factor $\CIF_q$ ($q \geq 1$) defined below:
\bel{CIF}
\CIF_q = \CIF_q(\xi,S) := \inf\Big\{ \frac{|S|^{1/q}\|\bX^\top\bX\bu\|_\infty}{n\|\bu\|_q}: 
\|\bu_{S^c}\|_1 < \xi \|\bu_S\|_1 \Big\}. 
\eel
The quantity $\CIF_q$ and its sign-restricted version 
have appeared in \cite{YeZ10}, where invertibility factor-based 
$\ell_q$ error bounds of the form (\ref{th-1-1}) 
%in Theorem~\ref{thm:param-est} 
below have been proven to sharpen earlier results for the Lasso and Dantzig selector 
\cite{CandesTao07,ZhangHuang08,BickelRT09,ZhangT09,vandeGeerB09} when $q\in [1,2]$. 
Such error bounds are of optimal order \cite{YeZ10,RaskuttiWY09}.
Of special interests are $q\in [1,2]$ for which the condition $\CIF_q>0$ on $\bX$ 
is $\ell_2$ regular and 
\bel{factor-comparison}
%\CIF_1(\xi,S)\ge \max\Big\{\frac{\RE_1^2(\xi,S)}{(1+\xi)^2},\frac{\CIF_2(\xi,S)}{1+\xi}\Big\}
%\ge \frac{\RE_1(\xi,S)\RE_2(\xi,S)}{(1+\xi)^2}. 
\CIF_1(\xi,S)\ge \frac{\RE_1^2(\xi,S)}{(1+\xi)^2},\ 
\CIF_2(\xi,S)\ge \frac{\RE_1(\xi,S)\RE_2(\xi,S)}{(1+\xi)}
\ge \frac{\RE_2^2(\xi,S)}{(1+\xi)}. 
\eel
Thus, $\CIF_q>0$ is an $\ell_2$ regularity condition for $q\in[1,2]$. 

A main advantage of using invertibility factor is that 
for $q>2$, invertibility factors still yield $\ell_q$ error bounds of optimal order % in the form of (\ref{th-1-1}),
which match results in \cite{ZhangT09,YeZ10}. However, 
the sparse and restricted eigenvalues do not yield error bounds of optimal order 
due to the unboundedness of 
$\max_{\|\bu\|_2=1}\|\bu_S\|_q\|\bu_S\|_1/|S|^{1/q}$ in $|S|$. 

We shall point out that different $\ell_2$ regularity conditions are typically not 
equivalent since different norms are involved in the definitions of different quantities.
For instance, in a specific example given in \cite{BickelRT09,vandeGeerB09}, 
$\RE_1$ and $\CIF_2$, uniformly bounded from away from zero, yield $\ell_1$ and $\ell_2$ 
error bounds of optimal order respectively, but $\RE_2$ does not. 

In the above discussion, we focus on fixed sample conditions like $\RE_2>0$ and $\CIF_2>0$, 
which hold when $\rank(\bX)=p$. 
These conditions can be directly seen as $\ell_2$ regular from their 
existing lower bounds for $p>n$ such as those in \cite{BickelRT09,YeZ10}.  
The optimality of the order of the error bounds based on such 
quantities can be also stated as $\ell_2$ regularity conditions by comparing them with 
sparse eigenvalues. See Remark \ref{remark-RIF} for more discussion. 

\subsection{Previous Results}\label{sec:prev-results}
Among concave penalties, the $\ell_1$ penalty is the only convex one. 
Thus, the Lasso ($\ell_1$ regularization) 
is a special case of (\ref{eq:hbbeta}) with $\rho(t;\lam)=\lambda|t|$ \cite{Tibshirani96, ChenDS01}:
\bel{Lasso}
\hbbeta^{(\ell_1)} = \arg\min_{\bb \in \R^p} \left[ \frac{1}{2n} \|\bX \bb - \by\|_2^2 + \lam \|\bb\|_1 \right] .
\eel
As a function of $\lam$, the Lasso path $\hbbeta=\hbbeta^{(\ell_1)}(\lam)$ matches that of $\ell_1$ constrained quadratic programming. 
One may use the homotopy/Lars algorithm to compute the complete Lasso path for 
$\lam \in [0,\infty)$ \cite{OsbornePT00a, OsbornePT00b,EfronHJT04} or simply use a standard convex optimization algorithm
to compute the Lasso solution for a finite set of $\lam$.
The Dantzig selector, proposed  in \cite{CandesTao07},  is an $\ell_1$-minimization method related to the Lasso, which solves
\[
\hbbeta = \argmin_{\bb \in \R^p} \|\bb\|_1  \quad \text{subject to } \|\bX^\top (\bX \bb - \by)\|_\infty \leq \lam .
\]
It has analytical properties similar to that of Lasso, but can be computed by linear programming rather than quadratic programming
as in Lasso.
Analytic properties of the Lasso or Dantzig selector have been studied in 
\cite{KnightF00,GreenshteinR04,MeinshausenB06,Tropp06,ZhaoY06,Wainwright09,
CandesTao07,BuneaTW07,vandeGeer08,ZhangHuang08,MeinshausenY09,
BickelRT09,Koltchinskii09,ZhangT09,vandeGeerB09,CaiWX10,YeZ10}. 
A basic story is described in the following two paragraphs. 

Under various $\ell_2$ regularity conditions on $\bX$ and the 
$\ell_0$ sparsity condition on $\bbeta$, the Lasso and Dantzig selector control the estimation 
errors and the dimension of the selected model in the sense 
\bel{estimation-error}
\frac{\|\bX\hbbeta-\bX\bbeta\|_2^2}{\sigma^2\ln p}
+ \frac{\|\hbbeta-\bbeta\|_q^q}{\{(\sigma^2/n)\ln p\}^{q/2}}  + \|\hbbeta\|_0 
=O_P(s^*), \ 1\le q\le 2, 
\eel
\cite{CandesTao07,vandeGeer08,ZhangHuang08,
BickelRT09,Koltchinskii09,ZhangT09,YeZ10,BuhlmannGeer11}. 
Compared with the oracle $\obeta$ in (\ref{oracle-property}), 
the estimation loss of $\hbbeta$ is inflated 
by a factor of no greater order than $\sqrt{\ln p}$, and the size of the selected model is of 
the same order as the true one. 
When $\ln(p/n)\asymp \ln p$, it has been proved in \cite{YeZ10,RaskuttiWY09} that (\ref{estimation-error}) 
matches the order of the risk of a Bayes estimator for a class of (weak) signals close to zero, 
so that the order of this loss inflation factor $\sqrt{\ln p}$ is the smallest possible 
without further assumption on the strength of the signal $\bbeta$. 
This inflation factor can be viewed as the cost of not knowing $\supp(\bbeta)$. 
Nevertheless, when $\bbeta$ is strong (in the sense that its minimum nonzero coefficient is not close to zero),
then it is possible to achieve the oracle property, which removes the inflation factor.
However even in such cases, the logarithmic inflation is still present for the Lasso solution, 
and it is generally referred to as the {\em Lasso bias};
it means that the Lasso does not have the oracle property even when the signal is strong 
\cite{FanLi01,FanP04}. Nonconvex penalty can be used to remedy this issue.
For the Lasso and Dantzig selector, 
extensions of (\ref{estimation-error}) have been established 
for capped-$\ell_1$ sparse $\bbeta$ \cite{ZhangHuang08,ZhangT09,YeZ10} 
and  for $2<q\le \infty$ under certain $\ell_q$ regularity conditions on $\bX$ \cite{ZhangT09,YeZ10}. 
Error bounds of type (\ref{estimation-error}) have been used 
in the analysis of the joint estimation of the noise level $\sigma^*:=\|\bep\|_2/\sqrt{n}$ and $\bbeta$ 
\cite{StadlerBGeer10,Antoniadis10,SunZ10,SunZ11}. For example, the scaled Lasso 
\bes
\{\hbbeta,\hsigma\} = \argmin_{\{\bb,\sigma\}}\big\{\|\by-\bX\bb\|^2/(2n\sigma) 
+ 2\sqrt{(\ln p)/n}\|\bb\|_1\big\}
\ees
provides $|\hsigma/\sigma^*-1| =O_P(|S|(\ln p)/n)$ along with (\ref{estimation-error})
under $\ell_2$ regularity conditions \cite{SunZ11}.

%If $\{\hbbeta,\hsigma\}$ is the solution of the 
%joint convex minimization of $L_{\sigma\lam_0}(\bbeta)/\sigma$ over all $\{\bbeta,\sigma\}$ 
%with $\lam_0=2\sqrt{(\ln p)/n}$ and the $\ell_1$ penalty (scaled Lasso), 
%$|\hsigma/\sigma^*-1| =O_P(|S|(\ln p)/n)$ and (\ref{estimation-error}) still holds 
%under an $\ell_2$ regularity condition \cite{SunZ11}.

For variable selection, the Lasso is sign consistent in the event 
\bel{cond-Lasso-selection}
\sgn(\obeta)=\sgn(\bbeta),\quad 
\min_{j\in S}|\hbeta^{o}_j|\ge \theta_1^*\lam,\quad
\lam \ge \frac{\sigma\sqrt{(2/n)\ln(p-|S|)}}{(1 - \theta^*_2)_+},\quad 
\eel
where $\theta_1^*=\|(\bX_S^\top\bX_S/n)^{-1}\sgn(\bbeta_S)\|_\infty$, 
$\theta_2^*=\|\bX_{S^c}^\top\bX_S(\bX_S^\top\bX_S)^{-1}\sgn(\bbeta_S)\|_\infty$, $S=\supp(\bbeta)$, 
and $\obeta$ is the oracle estimator in (\ref{oracle-property}) 
\cite{MeinshausenB06,Tropp06,ZhaoY06,Wainwright09}. 
Since $\|\obeta-\bbeta\|_\infty=O_P(1)\sqrt{(\ln\|\bbeta\|_0)/n} = o_P(\lam)$ under mild conditions, 
$\theta^*_1$ and $\theta^*_2$ are key quantities in (\ref{cond-Lasso-selection}). 
For fixed $\kappa_0<1$, $\theta^*_2\le \kappa_0$ is called the 
neighborhood stability/strong irrepresentable condition \cite{MeinshausenB06,ZhaoY06}. 
For $\bX$ with iid $N(0,\bSigma)$ rows and given $S$, 
$\theta^*_1$ and $\theta^*_2$ are within a small fraction of their population versions 
with $\bSigma$ in place of $\bX^\top\bX/n$ \cite{Wainwright09}. 
For random $\bbeta$ with $\|\bbeta\|_0\lesssim n/\{\|\bX\bX^\top/p\|_2\ln p\}$ 
and uniformly distributed $\sgn(\bbeta)$ given $\|\bbeta\|_0$, 
$\theta^*_1\le 2$ and $\theta^*_2\le 1-1/\sqrt{2}$ with large probability 
under the incoherence condition 
$\max_{j\neq k}|\bx_j^\top\bx_k/n|\lesssim 1/(\ln p)$ \cite{CandesPlan09}. 
It is worth mentioning that neither the incoherence condition nor the strong irrepresentable condition
is $\ell_2$ regular: in fact they may both fail with $\theta^*_2 \asymp |S|^{1/2}$ 
and $\min_{j\in S}|\hbeta^{o}_j|\ge \theta_1^*\lam$ even in the classical setting of $\bX$ 
being rank $p$. Since $\theta^*_2\le 1$ is necessary for the selection consistency of 
the Lasso under the first two conditions of (\ref{cond-Lasso-selection}) \cite{Tropp06,Wainwright09}, 
this means that Lasso is not model selection consistent under $\ell_2$ regularity conditions.
In order to achieve model selection consistency under $\ell_2$ regularity, we have to employ a nonconvex penalty
in (\ref{eq:hbbeta}).

For sparse estimation, $\ell_0$ penalized LSE corresponds to the choice of 
$\rho(t;\lam) = \lam^2/2 I(t \neq 0)$ in (\ref{eq:hbbeta}),
and it was introduced in the literature \cite{Akaike73,Mallows73,Schwarz78} before Lasso. Formally, 
\begin{equation}
\hbbeta^{(\ell_0)} = \arg\min_{\bb \in \R^p} 
\left[ \frac{1}{2n} \|\bX \bb - \by\|_2^2 + \frac{\lam^2}{2} \|\bb\|_0 \right] .
\label{eq:reg-L0}
\end{equation}
This method is important for sparse recovery because with the Gaussian noise model 
$\bep \sim N(0,\sigma^2 I)$, uniform distribution on support set, and flat distribution of
$\bbeta$ within support, it is a Bayesian procedure for support set recovery. 
However, this penalty is not easy to work with numerically because it is discontinuous at zero.
The Lasso can be viewed as a convex surrogate of (\ref{eq:reg-L0}),
but it does not achieve model selection consistency under $\ell_2$ regularity,
nor does it have the oracle property when the signal is uniformly strong.

Continuous concave penalties other than Lasso have been introduced to remedy these problems. 
These concave functions approximate $\ell_0$ penalty better than Lasso, 
and thus can remove the Lasso bias problem. 
Most concave penalties are interpolations between the Lasso and the $\ell_0$ penalty. 
For example the $\ell_\alpha$ (bridge) penalty \cite{FrankF93} with $0<\alpha<1$
is equivalent to the choice of $\rho(t;\lam) = |t|^\alpha\lam^{2-\alpha}\{2(1-\alpha)\}^{1-\alpha}/(2-\alpha)^{2-\alpha}$
in (\ref{eq:hbbeta}). While the bridge penalty is continuous, its derivative is $\infty$ at $t=0$, which may still cause
numerical problems. In fact, the $\infty$ derivative value means that $\hbbeta=0$ is always a local solution
of (\ref{eq:hbbeta}) for bridge penalty, which prevents any possibility for the uniqueness of a reasonable local solution among sparse local solutions---
 a topic which we will investigate in this paper. 
In order to address this issue, additional penalty functions $\rho(t;\lam)$ with finite derivatives at $t=0$ have been suggested
in the literature, such as the SCAD penalty \cite{FanLi01}, and the MCP penalty \cite{Zhang10-mc+}.
These penalties can be written in a more general form as
$\rho(t;\lam) = \lam^2\rho(t/\lam)$ with 
$\rho(0)=0$ and $1-t\le(d/dt)\rho(t)\le 1$ for $t>0$, including the SCAD with 
$(d/dt)\rho(t) = 1\wedge(1 - (t-1)/(\gamma-1))_+$, $\gamma \ge 2$, and the MCP with 
$(d/dt)\rho(t) = 1\wedge(1 - t/\gamma)_+$, $\gamma \ge 1$.
It can be verified that the $\ell_\alpha$ penalty for $0\le\alpha\le 1$, the SCAD and MCP are all concave.
Another simple concave penalty is $\rho(t;\lam) = \min(\lam^2\gamma/2,\lam |t|)$, $\gamma \geq 1$, 
introduced in \cite{zhang09-multistage} as capped-$\ell_1$ penalty.
 
The above mentioned nonconvex interpolations of $\ell_0$ and $\ell_1$ penalties typically gain smoothness over the $\ell_0$ penalty and thus 
allow more computational options. Meanwhile, they may improve variable selection 
accuracy and gain oracle properties by reducing the bias of Lasso. 
A more direct way to reduce the bias of Lasso is via the adaptive Lasso procedure \cite{Zou06}, 
which solves the following weighted $\ell_1$ regularization problem for some $\alpha \in (0,1)$:
\[
\min_{\bb \in \R^p} \left[ \frac{1}{2n}\|\by -\bX\bb\|_2^2 + \lam \sum_{j=1}^p |\hat{w}_j|^{-\alpha} \; |b_j| \right] ,
\]
where $\hat{w}$ is an estimator of $\bbeta$ (for example, the solution of the standard unweighted
Lasso with regularization parameter $\lam$). 
A low-dimensional analysis in \cite{Zou06} showed that the Adaptive Lasso solution can achieve the oracle property
asymptotically. A high dimensional analysis of this procedure was given in \cite{HuangMZ08}. 
For variable selection consistency and oracle properties to hold, the adaptive Lasso requires stronger conditions in terms of the minimum signal strength 
$\min_{j \in \supp(\bbeta)} |\beta_j|$ than what is optimal. Specifically, the optimal requirement is
$\min_{j \in \supp(\bbeta)} |\beta_j|\geq \gamma \lam_{univ}$ with $\lam_{univ}=\sigma\sqrt{(2/n)\ln p}$ 
for some constant $\gamma$ that may depend on an $\ell_2$ regularity condition (also see Eq (\ref{min-beta}) below),
which can be achieved by other procedures \cite{Zhang10-mc+,ZhangTong11}; however,
adaptive Lasso requires $\min_{j \in \supp(\bbeta)} |\beta_j|$ to be significantly larger than the optimal order of $\lam_{univ}$.
This means adaptive Lasso is sub-optimal for sparse estimation problems.
We also observe that adaptive Lasso does not directly minimize a concave loss function, and hence it is not an instance of (\ref{eq:hbbeta}). 
It was later noted that this procedure is only one iteration of using the so-called MM (majorization-minimization) principle to solve
(\ref{eq:hbbeta}) with bridge penalty (for example, see \cite{ZouLi08}). The corresponding 
MM procedure is referred to as multi-stage convex relaxation in \cite{zhang09-multistage,ZhangTong11}.
For sparse estimation problem (\ref{eq:hbbeta}) with a penalty $\rho(t;\lam)$ that is
concave in $|t|$, this method iteratively invokes the solution of the following reweighted $\ell_1$ regularization problem
for stage $\ell=1,2,\ldots$, starting with the initial value of $\hbbeta^{(0)}=0$:
\[
\hbbeta^{(\ell)} = \argmin_{\bb \in \R^p} \left[ \frac{1}{2n}\| \bX \bb - \by\|_2^2 + 
  \sum_{j=1}^p  \lambda_j^{(\ell)} |b_j| \right] ,
\]
where $\lambda_j^{(\ell)} = (\partial/\partial t)\rho(t;\lam)|_{t=|\hbeta^{(\ell-1)}_j|}$ ($j=1,\ldots,p$).
This procedure may be regarded as a multi-stage extension of adaptive Lasso, which 
corresponds to the stage-2 solution $\hbbeta^{(2)}$ with bridge penalty.
Unlike results for adaptive Lasso, the results in 
\cite{zhang09-multistage,ZhangTong11} for the multistage relaxation method allow $\min_{j \in \supp(\bbeta)} |\beta_j|$ to achieve 
the optimal order of $\lam_{univ}$, which match those of  \cite{Zhang10-mc+} and improve upon
\cite{HuangMZ08}.
Moreover, only $\ell=O(\ln (\|\bbeta\|_0))$ stages is necessary
in order to achieve model selection consistency and oracle properties.
It is worth pointing out that the multi-stage procedure can also be adapted to work with
the Dantzig selector formulation \cite{LiWoYe10}.

For large $p$, the global solution of a nonconvex regularization method is  
hard to compute, so that local solutions are often used instead. 
Therefore theoretical analysis of nonconvex regularization has so far focused on specific numerical procedures that can find local solutions.
For the $\ell_0$ penalty, the penalized loss in (\ref{eq:hbbeta}) is typically evaluated 
for a subset of the $2^p$ possible models $\supp(\bb)$ such as those generated in 
stepwise regression. 
For smooth concave penalties, 
iterative algorithms can be used to find local minima of the penalized loss in (\ref{eq:hbbeta}) 
for a set of penalty levels 
\cite{HunterLi05, ZouLi08,zhang09-multistage,BrehenyH11,MazumderFH11,ZhangTong11}. 
For the MCP and other quadratic spline concave penalties,
a path following algorithm can be used to find local minima for an interval 
of penalty levels \cite{Zhang10-mc+}. 

Advances have been carried out in the analysis of nonconvex regularization methods 
in multiple fronts 
\cite{FanLi01,FanP04,Zou06,HuangMZ08, Zhang11-foba,Zhang10-mc+,zhang09-multistage,
BuhlmannGeer11}. 
For concave penalized loss in (\ref{eq:hbbeta}), local minimizers exist with the 
oracle property (\ref{oracle-property}) under mild conditions \cite{FanLi01,FanP04}. 
However, it remains unclear whether there exist computationally efficient procedures that can find local minimizers
investigated in \cite{FanLi01,FanP04}. 
For the MCP, the local minima generated by the path following 
algorithm controls the estimation error and model size in the sense of (\ref{estimation-error}) 
under an $\ell_2$ regularity condition on $\bX$ 
%when the maximum concavity of the penalty is smaller than $\kappa_-(m)$ for a certain $m\asymp \|\bbeta\|_0$ 
\cite{Zhang10-mc+}. Under the additional condition 
\bel{min-beta}
\min_{\beta_j\neq 0}|\hbeta^{o}_j|\ge \gamma \lam_{univ} \ge \sup\big\{t: (\pa/\pa t)\rho(t;\lam)\neq 0\big\}
\eel
with $\lam_{univ}=\sigma \sqrt{(2/n)\ln p}$ and a certain constant $\gamma > 1$, 
the same path following solution has the oracle property (\ref{oracle-property}) 
and thus the sign-consistency property \cite{Zhang10-mc+}.  
Similar results hold for the SCAD and certain other quadratic spline penalties \cite{Zhang10-mc+}.  
Under (\ref{min-beta}) and $\ell_2$ regularity conditions on $\bX$, 
the oracle property (\ref{oracle-property}) and model selection consistency has also been established 
for a specific forward/backward stepwise regression scheme 
\cite{Zhang11-foba} that can be regarded as an approximate $\ell_0$ penalty minimization algorithm. 
As we have mentioned earlier, the multi-stage relaxation scheme for minimizing (\ref{eq:hbbeta}) 
also leads to oracle inequality and  model selection consistency under
(\ref{min-beta}) and $\ell_2$ regularity conditions on $\bX$ \cite{zhang09-multistage,ZhangTong11}.

While a number of specialized results were obtained for specialized numerical procedures under appropriate conditions, it is not clear what are the relationship among these solutions. 
For example, it is not clear whether the global solution of (\ref{eq:hbbeta}) is unique and whether it corresponds to solutions of various numerical procedures studied in the literature. 
This leads to a conceptual gap in the sense that it is not clear
whether we should study specific local solutions as in the above mentioned previous work
or we should try to solve (\ref{eq:hbbeta}) as accurately as possible (with the hope of finding the global solution).
It is worth mentioning that related to this question, oracle inequalities involving global solutions with
nonconvex penalties have been studied in the literature 
(for example, see related sections in \cite{BuhlmannGeer11}).
However, such oracle inequalities do not lead to results comparable to those of 
\cite{Zhang10-mc+,zhang09-multistage,ZhangTong11}.
Another relevant study is \cite{KiChOh08}, which showed that in the lower dimensional scenario 
with $p\le n$, 
the global solution of (\ref{eq:hbbeta}) agrees with the oracle estimator $\obeta$ 
for the SCAD penalty when $\min_{\beta_j\neq 0}|\hbeta^{o}_j|$ is sufficiently large, and some other appropriate assumptions hold.
However, their analysis does not directly generalize to the more complex high dimensional setting.

The purpose of the remaining of this paper is to present some general results showing that under appropriate $\ell_2$-regularity conditions, the global solution of an appropriate nonconvex regularization method leads to desirable recovery performance; moreover, under suitable conditions, the global solution corresponds to the unique sparse local solution, which can be obtained via different numerical procedures. 
This leads to a unified view of concave high dimensional sparse estimation methods that can
serve as a guideline to develop additional numerical algorithms for concave regularization.

\section{High-Level Description of Main Results}

As we have discussed in our brief survey, concave regularized methods have been proven to 
control the estimation error and the dimension of the selected model (\ref{estimation-error}) 
under $\ell_2$ regularity conditions 
and possess the oracle property (\ref{oracle-property}) or the sign-consistency property  
under the additional assumption (\ref{min-beta}). However, these results are established 
for specific local solutions of (\ref{eq:hbbeta}) with specific penalties. For $p>n$ 
it is still unclear if the global 
minimizer in (\ref{eq:hbbeta}) is identical to these local solutions or controls estimation 
and selection errors in a similar way. 
In this paper, we unify the aforementioned results with the global solution of (\ref{eq:hbbeta}). 
Technical results are rigorously described in Section~\ref{sec:technical} below. This section explains the
main thrust of these results. 

We are mainly interested in two situations: $\ell_0$ regularization where $\rho(t;\lam)$ is 
discontinuous at
$t=0$, and smooth regularization which is continuous for all $t \geq 0$ and piece-wise differentiable. 
However, our basic results require only sub-additivity and monotonicity of $\rho(t;\lam)$ in $t$ in $[0,\infty)$. 

We shall first describe assumptions of our analysis in Subsection~\ref{sec:assumptions}.
As we have pointed out, the key regularity conditions required in our analysis are expressed in terms 
of the sparse eigenvalues
in (\ref{sparse-eigenvalues}) or invertibility factors $\RIF$ and $\CIF$ defined in (\ref{eq:RIF}) and (\ref{CIF}).
For the sake of clarity, we assume that these quantities are all constants, and this requirement is an $\ell_2$ regularity 
condition. Another condition required by our analysis is called {\em null-consistency}, which requires that 
if $\bbeta=0$, then the global minimizer of (\ref{eq:hbbeta}) is achievable at $\hbbeta=0$ 
(the actual condition, given in Assumption~\ref{assump:null}, is slightly stronger). 
Clearly this condition depends both on the matrix $\bX$ and on the noise vector $\bep$. It is shown 
in Subsection~\ref{sec:assumptions} that under the standard sub-Gaussian noise assumption
(see Assumption~\ref{assump:sub-Gaussian}), the null-consistency condition is $\ell_2$ regular.

In summary, all assumptions on $\bX$ needed in our analysis are $\ell_2$ regular; with this in mind, we may 
examine the main results, which are divided into four subsections. 

Subsection~\ref{sec:basic} is 
concerned with basic properties of global optimal solution of (\ref{eq:hbbeta}) for all subadditive 
nondecreasing penalties. Theorem~\ref{thm:param-est} gives $\ell_q$-norm error bounds for
$\|\hbbeta-\bbeta\|_q$ and a bound of the prediction error $\|\bX\hbbeta-\bX\bbeta\|_2$
that are comparable with known results
for $\ell_1$ regularization. This means that under appropriate $\ell_2$ regularity conditions, the global solution of 
concave regularization problems are no worse than the Lasso solution in terms of the order of 
estimation error. 
Theorem~\ref{thm:sparsity} shows that the global optimal solution of (\ref{eq:hbbeta}) is sparse, and under appropriate $\ell_2$ 
regularity conditions, the sparsity is of the same order as $\|\bbeta\|_0$; that is, $\|\hbbeta\|_0=O(\|\bbeta\|_0)$. 
Thus, (\ref{estimation-error}) holds for the global solution of (\ref{eq:hbbeta}). 
Moreover, if the second order derivative of $\rho(t;\lam)$ with respect to $t$ is sufficiently small, then the global
solution is also the unique sparse local solution of (\ref{eq:hbbeta}).
That is if a vector $\tbbeta$ is a local solution of (\ref{eq:hbbeta}) which is sparse: $\|\tbbeta\|_0=O(\|\bbeta\|_0)$,
then $\tbbeta$ is the global solution of (\ref{eq:hbbeta}).
None of these results require that $\min_{\beta_j\neq 0}|\hbeta^{o}_j|$ to be bounded away from zero. 
Furthermore, since these results require only $\ell_2$ regularity conditions, they apply to the 
case of $p\gg n$ as long as $s^*(\ln p)/n$ is small. 

Subsection~\ref{sec:L0} contains results specifically for $\ell_0$ regularization. First, the global solution 
of $\ell_0$ regularization is sparse. Moreover,  with sub-Gaussian noise, the prediction error bound for 
$\ell_0$ penalty in Theorem~\ref{thm:global-L0} does not depend on properties of the design matrix $\bX$. 
This significantly improves upon the corresponding result for general penalties 
in Theorem~\ref{thm:param-est}, which requires a non-trivial $\RIF_1$ condition on 
the design matrix $\bX$.
If the smallest sparse eigenvalue of $\bX^\top \bX/n$ is bounded from below, then we obtain in Theorem~\ref{thm:global-L0-2}  the selection consistency for $\ell_0$ regularization under  (\ref{min-beta}), which implies the oracle property.

Subsection~\ref{sec:local} considers penalties $\rho(t;\lam)$ which are both left- and right-differentiable, 
for which one can define (approximate) local solutions that are what numerical optimization procedures 
compute.
Theorem~\ref{thm:approx-local} considers the distance between two approximate local solutions. 
An immediate consequence of the result says that under appropriate assumptions, if $(\partial/\partial t) \rho(t;\lam)=0$ when $t$ is sufficiently large,
then there is a unique sparse local solution of (\ref{eq:hbbeta}) that corresponds to
the oracle least squares solution $\obeta$ under (\ref{min-beta}).
Therefore the unique local solution has the oracle property.
Moreover, this unique local solution has to be the global optimal solution according to Theorem~\ref{thm:sparsity}.
While Theorem~\ref{thm:approx-local} shows that it is possible for a penalty that is not second order differentiable to 
have a unique sparse local solution, it requires the condition (\ref{min-beta}) for such penalties. 
In contrast, with a second order differentiable
concave penalty, condition (\ref{min-beta}) is not needed in Theorem~\ref{thm:approx-local} for sparse local solutions to be unique. This suggests an advantage
for using smooth concave penalties which may lead to fewer local solutions under certain conditions.
Theorem~\ref{thm:obeta} gives sufficient conditions under which the global optimal solution 
of (\ref{eq:hbbeta}) achieves model selection consistency.
These sufficient conditions generalize the irrepresentable condition (\ref{cond-Lasso-selection})
for the model selection consistency of Lasso.
However, unlike the irrepresentable condition for Lasso, which is not an $\ell_2$ regularity condition, 
for a concave penalty where $(\partial/\partial t) \rho(t;\lam)$ is small for sufficiently large $t$, 
the generalized irrepresentable condition required in Theorem~\ref{thm:obeta} automatically holds 
when  $\min_{\beta_j\neq 0}|\hbeta^{o}_j|$ is not too small. Moreover, for appropriate nonconvex penalties,
it is possible to achieve a selection threshold of optimal order as in (\ref{min-beta}). 

Note that results in Subsection~\ref{sec:local} show that if one can find a local solution of (\ref{eq:hbbeta}) and the solution is sparse, 
then under appropriate conditions, it is the global solution of (\ref{eq:hbbeta}) and it is close to the oracle least
squares solution $\obeta$. It is possible to design numerical procedures that find a sparse local solution of (\ref{eq:hbbeta}).
For such a procedure, results of Subsection~\ref{sec:local} directly applies. 
Subsection~\ref{sec:approx} further develops along this line of thinking. Theorem~\ref{thm:local-global} shows that if a local solution
is also an approximate global solution, then it is sparse. This fact can be combined with 
results in Subsection~\ref{sec:local} to imply that under appropriate conditions, this particular local solution is the unique 
sparse local solution (which is also the global solution). Moreover, such a solution can be obtained via Lasso followed by gradient descent, as it can be shown that Lasso is a sufficiently accurate approximate global solution of (\ref{eq:hbbeta}) for the result to apply.

Our results essentially imply the following: under appropriate $\ell_2$ regularity conditions, plus appropriate assumptions on the penalty $\rho(t;\lam)$, 
procedures considered earlier such as MCP \cite{Zhang10-mc+} 
or multi-stage convex relaxation \cite{HunterLi05,ZouLi08,zhang09-multistage} give the same local solution that is also the
global minimizer of (\ref{eq:hbbeta}). Moreover, other procedures (such as Lasso followed by gradient descent) can be
designed to obtain the same
solution. Therefore these results present a coherent view of concave regularization by unifying a number of earlier approaches
and by extending a number of previous results. This unified theory presents a more satisfactory 
treatment of concave high dimensional sparse estimation procedures.

\section{Technical Statements of the Main Results} 
\label{sec:technical}
This section describes in detail our new technical results characterizing the global and local optimal solutions of (\ref{eq:hbbeta})
under different regularization conditions. 
Before going into the main results, we will specify some assumptions and definitions required in our analysis.

\subsection{General Assumptions and Definitions}
\label{sec:assumptions}
In this subsection, we describe and discuss general conditions imposed 
in the rest of the paper. 

We first consider conditions on the regularizer $\rho(t;\lam)$. 
We assume throughout the sequel the following conditions on the penalty function: 
\begin{itemize}
\item [(i)] $\rho(0;\lam)=0$; 
\item [(ii)] $\rho(-t;\lam) = \rho(t;\lam)$; 
\item [(iii)] $\rho(t;\lam)$ is non-decreasing in $t$ in $[0,\infty)$;
\item [(iv)] $\rho(t;\lam)$ is subadditive with respect to $t$,
$\rho(x+y;\lam)\le \rho(x;\lam)+\rho(y;\lam)$ for all $x, y \ge 0$. 
\end{itemize}
This family of penalties is closed under the summation and maximization 
operations and includes all functions increasing and concave in $|t|$. 
Although we are mainly interested in the case where $\rho(t;\lam)$ is concave in $|t|$, all of our results 
hold under the above specified weaker conditions, sometimes with side conditions such as the 
monotonicity of $\rho(t;\lam)/t$ for $t>0$ and the continuity of $\rho(t;\lam)$ at $t=0$. 
Therefore we will mention explicitly when such side conditions are needed.

We are particularly interested in the $\ell_0$ regularization $\rho(t;\lam)=(\lam^2/2) I(t \neq 0)$ which is 
discontinuous at $t=0$.
In addition, we are interested in regularizer $\rho(t;\lam)$ that is continuous in $t \geq 0$ 
and piece-wise differentiable. 
With such regularizers, local solutions of (\ref{eq:hbbeta}) can be defined as solutions with gradient zero. 
A local solution can be obtained using standard numerical procedures such as gradient descent. 
%Moreover, we show that under
%appropriate conditions, a sparse local solution is unique and correspond to the global optimal solution.

Given a regularizer $\rho(t;\lam)$ and any fixed $\lam>0$, 
we define the threshold level of the penalty as 
\begin{equation}\label{eq:t-rho}
\lam^* := \inf_{t>0}\{t/2+\rho(t;\lam)/t\}. 
\end{equation}
The quantity $\lam^*$ is a function of $\lam$ that provides a natural normalization of $\lam$.
We call $\lam^*$ the threshold level since 
$\argmin_t\{(z-t)^2/2 + \rho(t;\lam)\}=0$ iff $|z| \le \lam^*$. This can be easily seen from 
$(z-t)^2/2 + \rho(t;\lam) - z^2/2 = t\{t/2 + \rho(t;\lam)/t - z\}$. 
If $\rho(t;\lam)$ is continuous at $t=0$ and concave in $t\in (0,\infty)$, 
then $\lam^* \le \lim_{t\to 0+}(\pa/\pa t)\rho(t;\lam)$. 
For simplicity, we may also require that $\rho(t;\lam)$ be chosen such that $\lam^*=\lam$,
which holds for the penalties discussed in Subsection~\ref{sec:prev-results}, 
such as $\ell_0$, bridge, SCAD, MCP, and capped-$\ell_1$ penalties.

In the following and in the proofs, we will use the short-hand notation
\[
\|\rho(\bb;\lambda)\|_1 = \sum_{j=1}^p \rho(b_j; \lambda),\ \forall\  \bb=(b_1,\ldots,b_p)^\top. 
\]
%where vector $\bb=(b_1,\ldots,b_p)^\top$.

\begin{definition} 
The following quantity bounds a general penalty via $\ell_1$ penalty for sparse vectors: 
\begin{equation}
\Delta(a,k;\lam) = \sup \Big\{ \|\rho(\bb;\lam)\|_1 : \|\bb\|_1 \leq a k, \|\bb\|_0=k\Big\}. 
\label{eq:Delta}
\end{equation}
\end{definition}

\begin{proposition}\label{prop:concave-Delta} 
Let $\rho^*(t;\zeta) = \zeta |t| + (\zeta -|t|/2)_+^2/2$. Let $\lam^*$ be as in (\ref{eq:t-rho}). Then, 
\bes
&& \min\big\{\lam^* |t|/2, (\lam^*)^2/2\big\} \le \rho(t;\lam) \le \rho^*(t;\lam^*),\ 
\cr && \Delta(a,k;\lam)\le k\rho^*(a;\lam^*)\le k\max(a,2\lam^*)\lam^*.  
\ees
\end{proposition}

\begin{remark}\label{remark:rho}
It follows from Proposition \ref{prop:concave-Delta} that given a threshold level $\lam^*$, 
all penalty functions satisfying general conditions (i)-(iv) 
are bounded by a capped-$\ell_1$ penalty from below and the maximum of the 
$\ell_0$ and $\ell_1$ penalties from above, up to a factor of 2. 
The function $\rho^*(t;\zeta)$ is a convex quadratic spline fit of $\max(\zeta^2/2,\zeta |t|)$, 
the maximum of the $\ell_0$ and $\ell_1$ penalties with threshold level $\zeta$. 
\end{remark}

\begin{remark}\label{remark:Delta}
 A trivial upper bound is $\Delta(a,k;\lam)\le k\max_t\rho(t;\lam)$, which is useful 
only for bounded penalties. The $\ell_\infty$ bound $\rho(t;\lam) \le \gamma^*\lam^2$ holds with 
$\gamma^*=1/2$ for the $\ell_0$ penalty, 
$\gamma^*=\gamma/2$ for the capped-$\ell_1$ penalty and MCP, 
and $\gamma^*=(1+\gamma)/2$ for the SCAD penalty. 
If $\rho(t;\lam)$ is concave in $t\in[0,\infty)$, then $\Delta(a,k;\lam)\le k\rho(a;\lam)$ 
by the Jensen inequality. 
For $a\ge 2\lam^*$, $\Delta(a,k;\lam)\le a\lam^* k$ matches the trivial bound for the 
$\ell_1$ penalty, for which $\lam=\lam^*$. 
\end{remark}

Next, we consider conditions on the design matrix $\bX$. 
Recall that $\bX$ is column normalized to $\|\bx_j\|_2^2=n$ for simplicity.
Our analysis also depends on the sparse eigenvalues defined in (\ref{sparse-eigenvalues}) 
and the restricted invertibility factor defined as follows.

\begin{definition}
  For $q\ge 1$, $\xi>0$ and $S\subset \{1,\ldots,p\}$, we define the restricted invertibility factor as
  \bel{eq:RIF}
  \RIF_q(\xi,S) = \inf\Big\{ \frac{|S|^{1/q}\|\bX^\top\bX\bu\|_\infty}{n\|\bu\|_q}: 
  \|\rho(\bu_{S^c};\lam)\|_1 < \xi \|\rho(\bu_S;\lam)\|_1 \Big\}. 
  \eel
\end{definition}

The restricted invertibility factor is the quantity needed to separate conditions on $\bX$ and $\bep$ 
in our analysis. For $1\le q\le 2$,  sparse eigenvalues can be used to find lower bounds of 
$\RIF_q(\xi,S)$. 

\begin{proposition}\label{prop:invertibility} Let $\CIF$ be as in (\ref{CIF}). 
If $t/\rho(t;\lam)$ is increasing in $t\in(0,\infty)$, then 
\bel{RIF-CIF}
\RIF_q(\xi,S) \ge \inf_{|A|=|S|}\CIF_q(\xi,A).
\eel
\end{proposition}

\begin{remark} 
For the $\ell_1$ penalty, $\RIF_q=\CIF_q$. 
  If $\rho(t;\lam)$ is concave in $t \in [0,\infty)$, then $t/\rho(t;\lam)$ is increasing in $t$. 
  Thus, Proposition~\ref{prop:invertibility} is applicable to all penalty functions 
  discussed in Subsection~\ref{sec:prev-results}, including 
  the $\ell_0$, bridge, SCAD, MCP, and capped-$\ell_1$ penalties. 
\end{remark}

\begin{remark}\label{remark-CIF-sparse-eigen} 
The $\CIF$ can be uniformly bounded from below in terms of sparse eigenvalues: 
\bel{eq:CIF-sparse-eigen}
\CIF_q(\xi,S) 
\ge \frac{I\{1\le q\le 2\}\{\kappa_-(k+\ell)-(\xi/2)(k/\ell)^{1/2}\kappa_+(k+5\ell)\}}
{(1+\xi)^{2/q-1}(1+\xi^2k/(4\ell))^{1-1/q}(1+\ell/k)^{1/2}}, 
\eel
for all $1\le \ell\le (p-|S|)/5$ by Proposition 5 and (21) in \cite{YeZ10}, 
where $k=|S|$, and $\kappa_-(m)$ and $\kappa_+(m)$ are as in (\ref{sparse-eigenvalues}). 
For example, if we take $\xi=2$ and $\ell=2k$ and $q=2$, then
\[
\CIF_2(\xi,S) 
\ge \big\{\kappa_-(3k)-\kappa_+(11k)/\sqrt{2}\big\}\big/\sqrt{4.5} .
\]
\end{remark}

\begin{remark}\label{remark-RIF}
It follows from Proposition~\ref{prop:invertibility} and Remark \ref{remark-CIF-sparse-eigen} 
that conditions $\RIF_q(\xi,S)>0$ and $1/\RIF_q(\xi,S)=O(1)$ 
are both $\ell_2$-regularity conditions on $\bX$ for $1\le q\le 2$. Moreover, $\rank(\bX)=p$ implies 
$\RIF(\xi,S)>0$. To check the $\ell_2$ regularity of these conditions, 
we suppose that the rows of $\bX$ are iid from $N(0,\bSigma)$ with all eigenvalues of $\bSigma$ 
in $[c_1,c_2]\subset (0,\infty)$. Then, $c_1/2\le\kappa_-(m)$ and $\kappa_+(m)\le 2c_2$ 
with at least probability $1-\delta\in [0,1)$ for $m \le c_3 n/\ln(p/\delta)$ for a certain $c_3>0$. 
Let $c_4=\{c_1/(\xi c_2)\}^2$. 
In this event, setting $k=s^*$ and $\ell=(m-s^*)/5$ in (\ref{eq:CIF-sparse-eigen}) yields 
\bes
\min_{|S|\le s^*}\RIF_2(\xi,S)\ge \min_{|S|\le s^*}\CIF_2(\xi,S)
\ge (c_1/4)\big/\sqrt{(1+\xi^2c_4/4)(1+1/c_4)}
\ees 
when $5s^*/(m-s^*)< c_4$ for some $m \le c_3 n/\ln(p/\delta)$, 
which holds when $(s^*/n)\ln(p/\delta)\le c_3/(1+5/c_4)$. 
\end{remark}

Finally, we consider conditions on the error vector. 

\begin{assumption} \label{assump:sub-Gaussian}
  An error vector $\bep$ is sub-Gaussian with noise level $\sigma$ if for all $t \geq 0$:
  \[
  P\big( |\bu^\top \bep | > \sigma t \big) \le \exp (- t^2/2)
  \]
  for all vector $\bu$ with $\|\bu\|_2=1$ and 
  \bes
  P\big( \|\bP_A\bep\|_2/|A|^{1/2} > \sigma(1+ t)\big) \le \exp( - |A| t^2/2)
  \ees
  for all subsets $A\subset\{1,\ldots,p\}$, where $\bP_A$ is the orthogonal projection to 
  the range of $\bX_A$
  (that is, $\bP_A=\bX_A\bX_A^\dagger$, where $\bX_A^\dagger$ is the Moore-Penrose generalized 
  inverse of $\bX_A$).
  \end{assumption}

The above sub-Gaussian condition holds with $\bep\sim N(0,\sigma^2\bI_{n\times n})$. 
It is equivalent to  
the more common version of the sub-Gaussian condition $Ee^{\bv^\top\bep/\sigma'}\le e^{\|\bv\|^2/2}$ 
for all vectors $\bv$ and a constant $\sigma'$ of the same order as $\sigma$. 
As we have mentioned in Section 3, what we really need is a null-consistency 
condition, which we give below. The sub-Gaussian condition will be used to verify 
the null consistency condition. 

\begin{assumption} \label{assump:null}
  Let $\eta\in (0,1]$. We say that the regularization method
  (\ref{eq:hbbeta}) satisfies the $\eta$ {\em null-consistency} condition if the following equality holds:
  \begin{equation}
    \label{eq:null}
    \min_{\bb \in \R^p} \Big(\|\bep/\eta - \bX\bb\|_2^2/(2n)+\|\rho(\bb;\lam)\|_1\Big) = \|\bep/\eta\|_2^2/(2n)\Big . .
  \end{equation}
\end{assumption}

\begin{remark}
Given $\eta=1$, the null-consistency condition %(\ref{eq:null}) 
means that if $\bbeta=0$, then
the global minimizer of (\ref{eq:hbbeta}) is achievable at $\hbbeta=0$.
This requirement is clearly necessary for the global minimizer of (\ref{eq:hbbeta}) to 
satisfy the error bound (\ref{th-1-1}) in Theorem \ref{thm:param-est} below 
for $|S|=0$.
Here, we also allow a slightly stronger condition with $\eta<1$, which requires $\hbbeta=0$ for $\bbeta=0$ 
when the noise $\bep$ is proportionally inflated by $1/\eta$. 
\end{remark}

\begin{proposition}\label{prop-2} 
Suppose that $\bep$ is sub-Gaussian with noise level $\sigma$, 
$0<\delta\le 1$ and $\zeta_0>0$. 
Suppose $\rho(t;\lam)\ge \big((\lam^*)^2/2\big)\wedge(\lam^* |t|)$ with 
$\lam^* \ge (1+\zeta_0)(\sigma/\eta)n^{-1/2}\big(1+\sqrt{2\ln(2p/\delta)}\big)$.  
Then, (\ref{eq:hbbeta}) satisfies the
$\eta$ null-consistency condition with at least probability $2- e^{\delta/2}-\exp(-n(1-1/\sqrt{2})^2)$, 
provided that 
\bel{prop-2-1}
\max\left\{\lam_{\max}^{1/2}\big(\bX_B^\top \bP_A\bX_B/n): 
\begin{matrix} B\cap A=\emptyset, |A|=\rank(\bP_A)=|B|=k, 
\cr k(1+\zeta_0)^2(1+\sqrt{2\ln(2p/\delta)})^2\le 2n
\end{matrix}\right\} 
\le \zeta_0. 
\eel
Moreover, (\ref{prop-2-1}) holds with no smaller probability than $1 - \delta^4/(16p^2)$ 
if the rows of $\bX$ are iid from $N(0,\bSigma)$ and 
$\sqrt{8}\lam_{\max}^{1/2}(\bSigma)\le \zeta_0(1+\zeta_0)$. 
This means that under the sub-Gaussian condition on $\bep$, 
the $\eta$ null-consistency is an $\ell_2$-regularity condition.
\end{proposition}

\begin{remark}\label{remark:null-consist}
  The condition $\rho(t;\lam)\ge \min\big(\lam^2/2,\lam |t|\big)$ holds for the $\ell_0$, $\ell_1$, 
  SCAD, and capped $\ell_1$ penalties, so that Proposition \ref{prop-2} is directly 
  applicable with $\lam=\lam^*$. In general, the condition of Proposition \ref{prop-2} 
  holds for all penalties considered in this paper when the threshold level in (\ref{eq:t-rho}) 
  satisfies $\lam^*\ge 2(1+\zeta_0)(\sigma/\eta)n^{-1/2}\big(1+\sqrt{2\ln(2p/\delta)}\big)$, 
  in view of the lower bound of $\rho(t;\lam)$ in Proposition~\ref{prop:concave-Delta}. 
  For $\ell_0$ and $\ell_1$ penalties, we may set $\zeta_0=0$ in Proposition \ref{prop-2}   
  (the extra condition (\ref{prop-2-1}) is not necessary). The simplified condition for $\ell_0$ 
  penalty is explicitly given in Theorem~\ref{thm:global-L0}. For the $\ell_1$ penalty, the 
  $\eta$ null consistency condition is equivalent to $\|\bX^\top\bep\|_\infty \leq \eta \lambda n$. 
\end{remark}

\subsection{Basic Properties of the Global Solution}
\label{sec:basic}
We now turn our attention to the global solution of (\ref{eq:hbbeta}) with a general 
subadditive nondecreasing regularizer $\rho(t;\lam)$. 
We first consider the estimation of $\bX\bbeta$ and $\bbeta$. 

\begin{theorem}\label{thm:param-est} 
Let $S=\supp(\bbeta)$, $\hbbeta$ be as in (\ref{eq:hbbeta}), $\lam^*$ as in (\ref{eq:t-rho}), 
and $\RIF_q(\xi,S)$ as in (\ref{eq:RIF}). 
Consider $\eta \in (0,1)$, and $\xi = (\eta+1)/(1-\eta)$, and assume that (\ref{eq:null}) holds. Then
for all $q \geq 1$:
\bel{th-1-1}
\|\hbbeta-\bbeta\|_q \le (1+\eta) \lam^* |S|^{1/q}/\RIF_q(\xi,S), 
\eel
and with $a_1=(1+\eta)/\RIF_1(\xi,S)$ and $\Delta(a,k;\lam)$ in (\ref{eq:Delta}), 
\bel{th-1-2}
\|\bX\hbbeta-\bX\bbeta\|_2^2/n \le  2 \xi \Delta\Big(a_1\lam^*,|S|;\lam\Big)
\le 2\xi(a_1\vee 2)(\lam^*)^2|S|. 
\eel
\end{theorem}

By using the bound $\Delta\Big(a_1\lam^*,|S|;\lam\Big) \leq |S| \max_t \rho(t;\lam)$, we obtain
the following corollary.
\begin{corollary}\label{cor:pred} Consider penalties $\rho(t;\lam)$ indexed by the threshold level; 
$\lam^*=\lam$ in (\ref{eq:t-rho}). 
Suppose that the $\eta$ null consistency condition (\ref{eq:null}) holds. 
Let $S=\supp(\bbeta)$ and $\gamma^*=\max_t\rho(t;\lam)/\lam^2$. Then, 
\bes
\|\bX\hbbeta-\bX\bbeta\|_2^2/n \le 2\{(1+\eta)/(1-\eta)\}\gamma^*\lam^2|S|. 
\ees
In particular, $\gamma^*=1/2$ for the $\ell_0$ penalty $\rho(t;\lam)=(\lam^2/2)I(t\neq 0)$, 
$\gamma^*=\gamma/2$ for the capped-$\ell_1$ penalty $\rho(t;\lam)=(\lam^2\gamma/2)\wedge(\lam|t|)$ 
and the MCP $\rho(t;\lam)=\lam\int_0^{|t|}(1-x/(\lam\gamma))_+dx$, 
and $\gamma^*=(1+\gamma)/2$ for the SCAD penalty 
$\rho(t;\lam)=\lam\int_0^{|t|}\min\{1,(1-(x/\lam-1)/(\gamma-1))_+\}dx$. 
\end{corollary}

\begin{remark} 
It is worthwhile to note that the prediction error bound in Corollary \ref{cor:pred} does 
not depend on $\bX$, provided that penalty is large enough to guarantee 
null consistency. For the $\ell_0$ penalty, the null consistency requires only 
$\|\bx_j\|_2=\sqrt{n}$ on $\bX$, which we assume anyway.  For other concave penalties 
in Corollary \ref{cor:pred}, we are only able to provide null consistency in 
Proposition \ref{prop-2} under a mild condition on the upper eigenvalue 
of $\bX_B^\top\bP_A\bX_B/n$, but not on the sparse lower eigenvalue of the Gram matrix. 
%We note that when $\bep\sim N(0,\sigma^2\bI_{n\times n})$, a consistent estimate of $\sigma$ 
%can be obtained under proper conditions. 
\end{remark}

Next we provide an upper bound for the sparseness of $\hbbeta$ based on Theorem \ref{thm:param-est}
and the maximum sparse eigenvalue $\kappa_+(m)$. We denote by $\drho(t;\lam) = (\pa/\pa t)\rho(t;\lam)$ 
any value between the left- and right- derivatives of $\rho(\cdot;\lam)$ and assume the 
left- and right-differentiability of $\rho(\cdot;\lam)$ whenever the notation $\drho(t;\lam)$ is invoked. 
For example, if $\rho(t;\lam)=\lambda |t|$, then $\drho(0\pm;\lam)=\pm\lam$ and 
$\drho(0;\lam)$ can be any value in $[-\lam,\lam]$
(which in all of our results, can be chosen as the most favorable value unless explicitly mentioned otherwise).

\begin{theorem}\label{thm:sparsity} 
Let $\{S,\hbbeta,\lam^*,\eta,\xi,a_1\}$ and $\Delta(a,k;\lam)$ be as in Theorem \ref{thm:param-est}, 
and $\Shat=\supp(\hbbeta)$. Suppose that (\ref{eq:null}) holds. 
Consider $t_0\ge 0$ and integer $m_0\ge 0$ satisfying $m_0=0$ for $t_0=0$ and 
\bel{th-2-1}
\sqrt{2\xi \kappa_+(m_0) \Delta(a_1\lam^*,|S|;\lam)/m_0} + \|\bX^\top\bep/n\|_\infty
< \inf_{0<s<t_0}\drho(s;\lam)
\eel
for $t_0>0$.  Then, 
\bel{th-2-2}
|\Shat\setminus S| < m := m_0 + \left\lfloor \xi \Delta(a_1\lam^*,|S|;\lam)/\rho(t_0;\lam) \right\rfloor .
\eel
\end{theorem}

The $\eta$ null consistency implies $\|\bX^\top\bep/n\|_\infty\le\eta\lam^*$ 
by Lemma \ref{lem:lam} in Section 5. 
If $\rho(t;\lam)$ is concave in $t>0$, then the right-hand side 
of (\ref{th-2-1}) can be replaced by $\drho(t_0;\lam)$ 
and $\rho(t_0;\lam)\ge t_0\drho(t_0;\lam)$. These facts give 
the following corollary for $\ell_\infty$ bounded and $\ell_1$ penalties.  

\begin{corollary}\label{cor:sparsity} (i) Let $\rho(t;\lam)$ and $\gamma^*$ be as in 
Corollary \ref{cor:pred}. Suppose (\ref{eq:hbbeta}) is $\eta$ null consistent in the sense of 
(\ref{eq:null}) and $\drho(a_0\lam;\lam)\ge \lam(1-a_1/\gamma)$ for some $a_0>0$ and $a_1\ge 0$. 
If $m_0=\alpha|S|$ is an integer and $2\gamma^*\kappa_+(\alpha|S|)/\alpha 
< (1-a_1/\gamma-\eta)^2(1-\eta)/(1+\eta)$, 
then 
\bel{cor:sparsity-1}
|\Shat\setminus S| < m: = \Big(\alpha+\frac{\gamma^*/a_0}{1-a_1/\gamma}\Big)|S|. 
\eel
(ii) Let $\Shat^{(\ell_1)}=\supp(\hbbeta^{(\ell_1)})$ with the Lasso (\ref{Lasso}) and 
$\CIF_q$ as in (\ref{CIF}). 
In the event $\|\bX^\top\bep/n\|_\infty\le\eta\lam$, 
\bel{cor:sparsity-2}
\frac{2\kappa_+(\alpha|S|)/\alpha}{\CIF_1((1+\eta)/(1-\eta),S)} 
< \frac{(1-\eta)^3}{(1+\eta)^2} \ \Rightarrow\ |\Shat^{(\ell_1)}\setminus S| < m: =\alpha|S|. 
\eel
\end{corollary}

\begin{remark}\label{Lasso-sparsity} Theorem \ref{thm:sparsity} and Corollary \ref{cor:sparsity} 
imply that the global solution $\hbbeta$ in (\ref{eq:hbbeta}) is sparse under appropriate assumptions.
For $\ell_0$ regularization, we may take $m_0=t_0=0$ with the convention $\kappa_+(0)/0=0$ 
in (\ref{th-2-1}). 
The Lasso also satisfies the dimension bound 
$|\Shat\setminus S| < m\vee 1$ under the SRC: 
$\{\kappa_+(m+|S|)/\kappa_-(m+|S|)-1\}/(2-2a_0)\le m/|S|$ with an $a_0\in (0,1)$, 
provided that $\lam\ge (1+o(1))\{\kappa_+^{1/2}(m)/a_0\}\sigma\sqrt{(2/n)\ln p}$ 
\cite{Zhang10-mc+}. An advantage of (\ref{cor:sparsity-2}) is to allow an 
$\lam$ not dependent on the upper sparse eigenvalue of the design for sub-Gaussian $\bep$. 
\end{remark}

\begin{remark}\label{uniqueness} 
Let $\kappa^* = \sup_{0<s<t}\{\drho(t;\lam)-\drho(s;\lam)\}/(s-t)$ be the maximum concavity of the penalty. 
Suppose $\kappa_-(|S|+m+\mtil-2)> \kappa^*$. Then, the penalized loss $L_\lam(\bb)$ in 
(\ref{eq:hbbeta}) is convex in all models $\supp(\bb)=A$ with $|A\setminus S|\le m+\mtil-2$. 
This condition has been called sparse convexity \cite{Zhang10-mc+}. 
If $m$ is as in (\ref{th-2-2}) or (\ref{cor:sparsity-1})  
and $\tbbeta$ is a local solution of (\ref{eq:hbbeta}) with 
$\#\{j\not\in S:\tbeta_j\neq 0\} < \mtil$, then the local solution must be identical to the global 
solution. 
\end{remark}

\begin{remark}\label{global-path}
Consider penalties with $\lam^*=\lam$ which holds for all penalties discussed 
in Subsection~\ref{sec:prev-results}. Let $\eta\in (0,1)$ and $\lam_*>0$ be fixed. 
Suppose Theorem \ref{thm:sparsity} or Corollary \ref{cor:sparsity} is applicable 
with $m\le\alpha^*|S|$ for a fixed constant $\alpha^*$ and all $\lam \ge \lam_*$. 
Suppose in addition $\drho(t;\lam)$ is continuous 
in $1/\lam \in [0,1/\lam_*]$ uniformly in bounded sets of $t$. 
Under the sparse convexity condition $\kappa_-(|S|+m-1)\ge \kappa_*>0$, 
with the maximum concavity $\kappa^*$ in Remark \ref{uniqueness}, 
the global solution forms a continuous path in $\real^p$ as a function of 
$1/\lam \ge 1/\lam_*$. This path is identical to the output of the path following algorithm 
in \cite{Zhang10-mc+} if it starts with $\hbbeta=0$ at $1/\lam=0$. 
We will show in Theorem \ref{thm:local-global} that gradient algorithms beginning from the Lasso 
may also yield the global solution under the sparse convexity condition. 
\end{remark}

As a simple working example to illustrate Corollaries \ref{cor:pred} and 
\ref{cor:sparsity}, we consider the capped-$\ell_1$ penalty 
explicitly given in Corollary \ref{cor:pred}. 
Let $a_0= \gamma/2$ in Corollary \ref{cor:sparsity}. We find 
\bes
 \|\bX\hbbeta-\bX\bbeta\|_2^2/n\le \lam^2|S|\gamma (1+\eta)/(1-\eta),
\cr  \gamma \kappa_+(\alpha |S|)  \leq \alpha(1-\eta)^3/(1+\eta) \ \Rightarrow\ |\Shat\setminus S| < (\alpha+1) |S| .
\ees
The MCP, also explicitly given in Corollary \ref{cor:pred}, 
provides the same prediction bound and 
\[
 \gamma \kappa_+(\alpha |S|)  \leq \alpha(2/3-\eta)^2(1-\eta)/(1+\eta) \ 
 \Rightarrow\ |\Shat\setminus S| < (\alpha+9/4) |S|
\]
by the same calculation with $\alpha_0=\gamma/3$. 
Note that generally speaking, unless stronger conditions are imposed,
Theorem \ref{thm:sparsity} 
only implies that $|\Shat\setminus S| =O(|S|)$ but not  $|\Shat\setminus S|=0$ required for 
model selection consistency. The model selection consistency will be studied later in the paper.

\subsection{The Global Solution of $\ell_0$ Regularization}
\label{sec:L0}
This subsection considers  the global optimal solution $\hbbeta^{(\ell_0)}$ of $\ell_0$ regularization in (\ref{eq:reg-L0}).
Our first result says that under appropriate conditions, this solution is sparse.
\begin{theorem} \label{thm:global-L0}
  If for all  $\bb \in \R^p$: $\bep^\top \bX \bb \leq \lam \eta \sqrt{n \|\bb\|_0} \|\bX \bb\|_2$
  for some $\eta <1$, then (\ref{eq:reg-L0}) satisfies the $\eta$ null-consistency condition.
 It implies that the global optimal solution of (\ref{eq:reg-L0}) satisfies
  \[  \|\hbbeta^{(\ell_0)}\|_0 \leq  \frac{1+\eta^2}{1-\eta^2} \|\bbeta\|_0 , \qquad
  \|\bX\hbbeta^{(\ell_0)}-\bX\bbeta\|_2^2 \leq  \frac{(1+\eta)\lam^2 \|\bbeta\|_0 }{1-\eta}.
  \]
\end{theorem}

We also have the following result about model selection quality for $\ell_0$ regularization.

\begin{theorem} \label{thm:global-L0-2}
  Assume that the assumption of Theorem~\ref{thm:global-L0} holds.
  Let $s=2 \|\bbeta\|_0/(1-\eta^2)$ and $\obeta$ be as in (\ref{oracle-property}). 
  Suppose $\|\bX^\top (\bP_S \bep - \bep)\|_\infty/n \leq \sqrt{2\kappa_-(s)} \lam$,
  where $\bP_S$ is the orthogonal projection to  the range of $\bX_S$.
  Let   $S=\supp(\bbeta)$, 
  $\delta^{o}=\#\{j\in S: |\hbeta^{o}_j| < \lam\sqrt{2/\kappa_-(s)}\}$, and $\Shat=\supp(\hbbeta^{(\ell_0)})$. Then,  
  \[
  |S-\Shat| + 0.5 |\Shat - S| \leq 2\delta^{o} , \qquad   \|\bX(\hbbeta^{(\ell_0)}-\obeta)\|_2^2 \leq 2 \lam^2 \delta^{o} .
  \]
\end{theorem}

If the error $\bep$ is sub-Gaussian in the sense of Assumption \ref{assump:sub-Gaussian}, 
then the condition of Theorems \ref{thm:global-L0} and \ref{thm:global-L0-2} 
holds with at least probability $2-e^\delta$ for 
$\lam \ge (\sigma/\eta)(1+\sqrt{2\ln(p/\delta)})/\sqrt{n}$.
Theorem \ref{thm:global-L0-2} implies that model selection consistency can be achieved
if the condition $\min_{j \in \supp(\bbeta)} |\hbeta^{o}_j| \geq \lam/\sqrt{\kappa_-(s)}$ holds, which
implies that $\delta^{o}=0$.  

\subsection{Approximate Local Solutions} % of Piecewise Differentiable Penalty}
\label{sec:local}
We have shown in Theorem~\ref{thm:sparsity}
that under appropriate conditions, the global solution of (\ref{eq:hbbeta}) is sparse. 
If $\rho(t;\lam)$ is both left- and right-differentiable, one can define the concept of 
local solution as follows. Given an excess $\nu\ge 0$, a vector $\tbbeta \in \R^p$ is an approximate 
local solution of (\ref{eq:hbbeta}) if 
\bel{approx-local}
\|\bX^\top (\bX \tbbeta - \by)/n + \drho(\tbbeta;\lam)\|_2^2 \le \nu. 
\eel
This $\tbbeta$ is a local solution if $\nu=0$. 
Note that by convention, $\drho(t;\lam)$ can be chosen to be any value
between $\drho(t_-;\lam)$ and $\drho(t_+;\lam)$ to satisfy the equation.
In this subsection, we provide estimates of distances between approximate local solutions 
and use them to prove the equality of oracle approximate local and global solutions of (\ref{eq:hbbeta}). 
This gives the selection consistency of the global solution studied in Subsection 4.2. 
The oracle LSE is considered as an approximate local solution. 
In addition, we define a sufficient condition for the existence 
of a sign consistent local solution which generalizes the irrepresentable condition for 
Lasso selection and becomes an $\ell_2$ regularity condition on $\bX$ for a broad class of 
concave penalties. 

We first provide estimates of distances between approximate local solutions. 
We use the following function $\theta(t,\kappa)$ to measure the degree of nonconvexity of 
a regularizer $\rho(t;\lam)$ at $t\in\R$. 
\begin{definition}
  For $\kappa \geq 0$ and $t \in \R$, define
  \[
  \theta(t,\kappa):= \sup_s \{ -\sgn(s-t) (\drho(s;\lam) - \drho(t;\lam)) - \kappa |s-t|\} .
  \]
  Moreover, given $\bu =(u_1,\ldots,u_p)^\top\in \R^p$, 
  we let $\theta(\bu,\kappa)=[\theta(u_1,\kappa),\ldots,\theta(u_p,\kappa)]$.
  %Assume $\rho(t;\lam)$ is left- and right-differentiable in $t \in [0,\infty)$.
  %For each $\kappa \geq 0$ and $t \in \R$, we define $\theta(t,\kappa)$ as
  %$\theta(t,\kappa)= \sup \{ -\sgn(s-t) (\drho(s;\lam) - \drho(t;\lam)) - \kappa |s-t| : s \in \R \}$.
\end{definition}

We are mostly interested in values of $\theta(t,\kappa)$ that achieves zero. 
We note that $\theta(t,\kappa)=0$ for convex $\rho(t,\lam)$ with $\kappa \geq 0$.
More generally, let $\kappa^*$ be the maximum concavity as in Remark~\ref{uniqueness}. 
Then, $\theta(t,\kappa)=0$ for all $t$ iff $\kappa \ge\kappa^*$. 
For $\drho(t+;\lam)<\drho(t-;\lam)$, $\theta(t,\kappa)>0$ for all finite $\kappa$. 
However, we only need $\theta(t,\kappa)=0$ for a proper set of $t$ in our selection 
consistency theory. As an example, for $\kappa=2/\gamma$, 
the capped-$\ell_1$ penalty $\rho(t;\lam)=\min(\gamma \lam^2/2,\lam |t|)$ 
gives $\theta(t,\kappa)=0$ when either $t=0\pm $ or $|t| \geq \gamma\lam$. 

The following theorem shows that under appropriate assumptions,
two sparse approximate local solutions $\tbeta^{(1)}$ and
$\tbeta^{(2)}$ are close.

\begin{theorem}\label{thm:approx-local}
Let $\tbbeta^{(j)}$ be approximate local solutions with excess $\nu^{(j)}$ and 
$\bDelta=\tbbeta^{(1)}-\tbbeta^{(2)}$. 
Let $\kappa_\pm(\cdot)$ be the sparse eigenvalues in (\ref{sparse-eigenvalues}) 
and $\Stil^{(j)}:=\supp(\tbbeta^{(j)})$. Consider any $S\subset\{1,\ldots,p\}$ with $k=|S|$, 
integer $m$ such that $m+k\ge |\Stil^{(1)}\cup\Stil^{(2)}|$, and 
$0<\kappa<\kappa_-(m+k)$. Then, 
\bel{th-5-1}
\|\bX\bDelta\|_2^2/n
\le \frac{2\kappa_-(m+k)}{(\kappa_-(m+k)-\kappa)^2}
\left\{\|\theta(|\tbbeta^{(1)}_{\Stil^{(1)}}|,\kappa)\|_2^2 
+ |\Stil^{(2)}\setminus\Stil^{(1)}|\theta^2(0+,\kappa) + \nu\right\}
\eel
with $\nu = \{(\nu^{(1)})^{1/2}+(\nu^{(2)})^{1/2}\}^2$, and
\bel{th-5-2} 
|S \setminus \Stil^{(2)}| \leq 
\inf_{\lam_0>0} \left[ \#\left\{j \in S: |\tbeta_j^{(1)}| < \lam_0/\sqrt{\kappa_-(m+k)}\right\}
+ \|\bX\bDelta\|_2^2/(\lam_0^2 n)   \right]. 
\eel
If in addition $\theta(0+,\kappa)=0$ and 
$\drho(0+;\lam) > \|\bX_{S^c}^\top(\bX\tbbeta^{(1)}-\by)/n\|_\infty$ with 
$S\supseteq S^{(1)}$ and $|S|\ge k$, then
\bel{th-5-3}
\big|\Stil^{(2)}\setminus S\big| 
\le \frac{3\big[\big\{\kappa^2/\kappa_-(m+k)+\kappa_+(m)\big\}\|\bX\bDelta\|_2^2/n+\tnu^{(2)}\big]}
{\big\{\drho(0+;\lam) - \|\bX_{S^c}^\top(\bX\tbbeta^{(1)}-\by)/n\|_\infty\big\}^2}. 
\eel
\end{theorem}

Let $S=\supp(\bbeta)$. 
For comparison between a sparse local or global solution $\tbbeta^{(2)}$ 
with $|\Stil^{(2)}\setminus S|\le m$ and an oracle solution $\tbbeta^{(1)}$ with $\Stil^{(1)}=S$, 
the sparse convexity condition implies $\tbbeta^{(2)}=\tbbeta^{(1)}$ when 
$\kappa^*<\kappa_-(|S|+m)$ as in Remark~\ref{uniqueness}. However, since 
$\kappa^*=\infty$ when $\drho(t+;\lam)<\drho(t-;\lam)$ at a point $t>0$, 
the sparse convexity argument requires the continuity of $\drho(t;\lam)$ for $t>0$. 
This does not apply to the capped-$\ell_1$ penalty. 
In Theorem \ref{thm:approx-local}, 
if $\theta(0+,\kappa)=\theta(\tbbeta^{(1)}_S,\kappa)=0$ with $\kappa < \kappa_-(|S|+m)$,
then $\bX\bDelta=0$, and hence $\tbbeta^{(2)}=\tbbeta^{(1)}$
(since $\kappa_{-}(|\Stil^{(1)}\cup\Stil^{(2)}|) >0$).
Thus, the sparse convexity condition is much weakened to cover all left- and right-differentiable 
penalties such as the capped-$\ell_1$.  
On the other hand, Theorem \ref{thm:approx-local} does not weaken the sparse convexity 
condition for the MCP $\rho(t;\lam)=\lam\int_0^{|t|}(1-x/(\gamma\lam))_+dx$, 
for which $\theta(0+;\kappa)=0$ iff $\kappa\ge\kappa^*=1/\gamma$ iff $\theta(t;\kappa)=0$ for all $t>0$. 
It is worth pointing out that for a piecewise differentiable penalty that is not second order 
differentiable, the condition $\theta(\tbbeta^{(1)}_S,\kappa)=0$ (thus, the uniqueness of local solution)
typically requires $|\tbeta_j^{(1)}|$ to be large to avoid the discontinuities of $\drho(t;\lam)$ when $j \in S$. 
As pointed out in Remark~\ref{uniqueness}, this is not necessary 
when the penalty is second order differentiable. This means that there can be advantages of 
using smooth penalty terms that may have fewer local minimizers under certain conditions.

As a simple working example to illustrate Theorem~\ref{thm:approx-local}, 
we consider the capped $\ell_1$ penalty of the form $\rho(t;\lam)=\min(\gamma \lam^2/2, \lam |t|)$.
Let $S=\supp(\bbeta)$. Assume that $\kappa = \kappa_-(m+|S|)/2 \geq 2/\gamma$. Then
$\theta(t,\kappa)=0$ when either $t=0\pm $ or $|t| \geq \gamma\lam$. 
Therefore, if we define $\tbbeta^{(1)}$ as $\tbeta_j^{(1)}=\hbeta^{o}_j$
when $|\hbeta^{o}_j| \geq \gamma \lam$ and $\tbbeta^{(j)}=0$ otherwise, then
\[
\|\bX\bDelta\|_2^2/n
\le \frac{8\nu }{\kappa_-(m+|S|)} ,
\]
and by taking $\lam_0=\gamma \lam \sqrt{\kappa_-(m+|S|)}$, we have
\[
|S \setminus \Stil^{(2)}| \leq \frac{\|\bX\bDelta\|_2^2}{\gamma^2 \lam^2 \kappa_-(m+|S|) n}   ,
\quad
\big|\Stil^{(2)}\setminus S\big| 
\le \frac{3\big[1.25 \kappa_+(m)\|\bX\bDelta\|_2^2/n+\tnu^{(2)}\big]}
{\big\{\lam - \|\bX_{S^c}^\top(\bX\tbbeta^{(1)}-\by)/n\|_\infty\big\}^2}. 
\]

We now consider selection consistency of the global solution (\ref{eq:hbbeta}) by comparing 
it with an oracle solution with Theorem \ref{thm:approx-local}. 
For this purpose, we treat the oracle LSE as an approximate local solution by finding its excess $\nu$ in 
(\ref{approx-local}), and provide a sufficient condition for the existence of a sign consistent 
oracle local solution. This sufficient condition is characterized by the following extension of the 
quantities $\theta^*_1$ and $\theta^*_2$ in (\ref{cond-Lasso-selection}) from the $\ell_1$ to general penalty: 
\begin{align*}
\theta_1 =& \inf\big\{\theta: \|(\bX_S^\top\bX_S/n)^{-1}\drho(\bv_S+\obeta_S;\lam)\|_\infty \le \theta\lam^*,\ 
\forall \|\bv_S\|_\infty\le \theta\lam^*\big\}, \\
\theta_2 =& \sup\big\{ 
\|\bX_{S^c}^\top\bX_S(\bX_S^\top\bX_S)^{-1}\drho(\bv_S+\obeta_S;\lam)\|_\infty/\lam^*: 
\|\bv_S\|_\infty\le \theta_1\lam^*\big\} ,
\end{align*}
where $S=\supp(\bbeta)$ and $\obeta$ is the oracle LSE in Definition \ref{def:terminology} (d). 
Note that when $\drho(\obeta_S;\lam)=0$, $\theta_1=0$ is attained with 
$\bv_S=0$ and consequently $\theta_2=0$. 

\begin{theorem} \label{thm:obeta} 
(i)  Let $S=\supp(\bbeta)$ and $\bP_S$ be the projection to the column space of $\bX_S$. 
Suppose $\rho(t;\lam)$ is left- and right-differentiable in $t >0$ and 
$\|\bX_{S^c}^\top\bP_S^\perp\bep\|_\infty\le\drho(0+;\lam)$.  
Then, the oracle LSE $\obeta$ satisfies (\ref{approx-local}) with 
$\nu=\|\drho(\obeta_S;\lam)\|^2$. If in addition (\ref{eq:null}) holds and 
$\nu=0=\theta(\obeta,\kappa)$ with a certain $\kappa<\kappa_-(m+|S|)$ 
and $m$ in (\ref{th-2-2}) or (\ref{cor:sparsity-1}), then $\obeta$ is the global solution of (\ref{eq:hbbeta}). \\
(ii) Suppose $\drho(t;\lam)$ is uniformly continuous in $t$ in the region 
$\cup_{j\in S}[\hbeta^{o}_j-\theta_1,\hbeta^{o}_j+\theta_1]$. Suppose 
\bel{th-3-1}
\sgn(\obeta)=\sgn(\bbeta),\quad
\min_{j\in S}|\hbeta^{o}_j| > \theta_1\lam^*,\quad \lam^*\ge \|\bX^\top \bP_S^\perp \bep /n\|_\infty/(1-\theta_2)_+\ .  
\eel
Then, there exists a local solution $\tbbeta^{o}$ of (\ref{eq:hbbeta}) satisfying 
$\sgn(\tbbeta^{o})=\sgn(\bbeta)$ and $\|\tbbeta^{o}-\obeta\|_\infty\le\theta_1\lam^*$. 
If in addition (\ref{eq:null}) holds and $\theta(\tbbeta^{o},\kappa)=0$ 
with a certain $\kappa<\kappa_-(m+|S|)$ 
and $m$ in (\ref{th-2-2}) or (\ref{cor:sparsity-1}). Then, $\tbbeta^{o}$ is the global solution of (\ref{eq:hbbeta}). 
\end{theorem}

\begin{remark} (i) For the capped-$\ell_1$ penalty $\rho(t;\lam)=\min(\gamma\lam^2/2,\lam|t|)$, 
$\nu=0=\theta(\obeta,\kappa)$ for $\kappa\ge 2/\gamma$ when 
$\min_{j\in S}|\hbeta^{o}_j|>\gamma\lam$. 
For the MCP $\rho(t;\lam)=\lam\int_0^{|t|}(1-x/(\gamma\lam))_+dx$, 
$\theta(\cdot,\kappa)=0$ for $\kappa\ge 1/\gamma$. 
For the SCAD penalty $\rho(t;\lam)=\lam\int_0^{|t|}\min\{1,(1-(x/\lam - 1)/(\gamma-1))_+\}dx$, 
$\theta(\cdot,\kappa)=0$ for $\kappa\ge 1/(\gamma-1)$. 
(ii) For the $\ell_1$ penalty, $\drho(\bb)=\sgn(\bb)$ so that (\ref{th-3-1}) is 
identical to (\ref{cond-Lasso-selection}) for the Lasso selection consistency. 
For concave penalties, $|\drho(t;\lam)|$ is small for large $|t|$, so that $\{\theta_1,\theta_2\}$ 
are typically smaller than $\{\theta_1^*,\theta_2^*\}$ for strong signals. In such cases, 
(\ref{th-3-1}) is much weaker than (\ref{cond-Lasso-selection}). 
\end{remark}

For a nonconvex penalties such that $\drho(t;\lam)=0$ when $|t| > a_0\lam$ for some constant 
$a_0>0$, we automatically have $\drho(\obeta_S;\lam)=0$ when 
$\min_{j \in S} |\hbeta^{o}_j| > a_0 \lam$, which implies that $\theta_1=\theta_2=0$.
This special case gives the following easier to interpret corollary as a direct consequence of 
Theorems~\ref{thm:approx-local} and \ref{thm:obeta}. 

\begin{corollary} \label{cor:model_sel}
 Let $S=\supp(\bbeta)$ and $\bP_S$ be the projection to the column space of $\bX_S$. 
 Suppose $\rho(t;\lam)$ is left- and right-differentiable in $t >0$ and 
 $\|\bX_{S^c}^\top\bP_S^\perp\bep/n\|_\infty\le\drho(0+;\lam)$. 
 If (\ref{eq:null}) holds and $\drho(\obeta_S;\lam)=0$, and 
 $\theta(\obeta,\kappa)=0$ with a certain $\kappa<\kappa_-(m+|S|)$ and $m$ 
 in (\ref{th-2-2}) or (\ref{cor:sparsity-1}), then
 $\obeta$ is the global solution of (\ref{eq:hbbeta}). 
 Moreover, for any other exact local solution $\tbbeta$ of (\ref{eq:hbbeta}) that is sparse
 with $|\sup(\tbbeta) \setminus S| \leq m$, we have $\tbbeta=\obeta$. 
\end{corollary}

Consider the simple examples of the capped-$\ell_1$ penalty and MCP. 
For the capped-$\ell_1$ penalty $\rho(t;\lam)=\min(\gamma \lam^2/2, \lam |t|)$, 
we pick a sufficiently large $\gamma$ such that $\gamma > 2/\kappa_-(|S|+m)$ 
for the $m$ in (\ref{th-2-2}) or (\ref{cor:sparsity-1}). This will be possible with $m\asymp |S|$ when 
$\kappa_-(m)$ is uniformly bounded away from zero for small $m(\ln p)/n$ and 
$|S|(\ln p)/n$ is even smaller. 
For the MCP $\rho(t;\lam)=\lam\int_0^{|t|}(1-x/(\lam\gamma))_+dx$, we pick 
$\gamma > 1/\kappa_-(|S|+m)$ for the $m$ in (\ref{th-2-2}) or (\ref{cor:sparsity-1}). 
If $\min_{j\in S}|\hbeta^{o}_j| \geq \gamma \lam$, then the conditions of Corollary~\ref{cor:model_sel}
are automatically satisfied for both penalties when  $\|\bX_{S^c}^\top\bP_S^\perp\bep/n\|_\infty < \lam$
(which can always be satisfied with a sufficiently large choice of $\lam$).
It follows that in this case, $\obeta$ is the global solution of (\ref{eq:hbbeta}), 
and there is no other local solution with no more than $m$ nonzero-elements out of $S$.
The essential condition here is the null consistency (\ref{eq:null}), which is an $\ell_2$ condition. 
Note that in view of Corollary \ref{cor:sparsity}, 
the RIF condition is not essential for the equality of the global and oracle solutions 
in these examples, both with finite $\gamma^*=\gamma/2$. A similar result hold for the SCAD penalty, 
with somewhat different constant factors. 
The requirement of $\min_{j\in S}|\hbeta^{o}_j| \geq \gamma \lam$ is natural, 
and it directly follows (with probability $1-\delta$) from
the condition of $\min_{j \in S} |\beta_j| > \gamma\lam + 
\sigma(1+\sqrt{2\ln(|S|/\delta)})\lam_{\min}^{-1/2}(\bX_S^\top\bX_S)$ under Assumption \ref{assump:sub-Gaussian}.

\subsection{Approximate Global Solutions}
\label{sec:approx}
We have mentioned in Remark \ref{global-path} that gradient algorithm from the Lasso 
may yield the global solution of (\ref{eq:hbbeta}) for general $\rho(t;\lam)$ 
under a sparse convexity condition or its generalization. Here we provide sufficient conditions for this to happen.  
This is done via a notion of approximate global solution. 
Given $\nu \geq 0$ and $\bb \in \R^p$,
we say that a vector $\tbbeta \in \R^p$ is a $\{\nu,\bb\}$ approximate global solution of (\ref{eq:hbbeta}) if
\begin{equation} \label{eq:approx-global}
\left[ \frac{1}{2n} \|\bX \tbbeta - \by\|_2^2 +  \|\rho(\tbbeta;\lam)\|_1 \right] 
- \left[ \frac{1}{2n} \|\bX \bb - \by\|_2^2 + \|\rho(\bb;\lam)\|_1 \right] \leq \nu. 
\end{equation}
To align different penalties at the same threshold level, we assume throughout this subsection 
that $\lam^*$ depends on $\rho(t;\lam)$ only through $\lam$ in (\ref{eq:t-rho}), e.g. $\lam^*=\lam$. 

One method to find sparse local solution is to find a local solution that is also an approximate global solution.
This can be achieved with the following simple procedure. First, we find the Lasso solution $\hbbeta^{(\ell_1)}$ of (\ref{Lasso}). 
The following theorem shows that it is a $\{\nu,\bbeta\}$ approximate global solution of (\ref{eq:hbbeta}) 
with a relatively small $\nu$ under proper conditions. 
Now we can start with this solution $\hbbeta^{(\ell_1)}$ and use gradient descent
to find a local solution $\tbbeta$ of (\ref{eq:hbbeta}) that is also an approximate global solution.
The following theorem then shows that under appropriate conditions, this local solution is sparse. 
Therefore results from Subsections 4.2 and 4.4 can be applied to relate it to the true global solution 
of (\ref{eq:hbbeta}). 

\begin{theorem} \label{thm:local-global} Consider a penalty functions $\rho(t;\lam)$ 
with $\lam=\lam^*$ in (\ref{eq:t-rho}). Suppose the $\eta$ null consistency 
condition (\ref{eq:null}) for $\rho(t;\lam)$ with $0<\eta<1$. \\
(i) Suppose $m=O(|S|)$ in (\ref{cor:sparsity-2}) or under the SRC in Remark \ref{Lasso-sparsity} 
for the Lasso $\hbbeta^{(\ell_1)}$ in (\ref{Lasso}). 
Then, the Lasso $\hbbeta^{(\ell_1)}$ is a $\{\nu,\bbeta\}$ approximate global solution for the 
penalty $\rho(t;\lam)$ with $\nu\lesssim \lam^2|S|$. \hspace{0.1in} \\
(ii) Assume that $\rho(t;\lam)$ is continuous at $t=0$.
   Let $\tbbeta$ be an local solution of (\ref{eq:hbbeta}) that is also a $\{\nu,\bbeta\}$ approximate 
   global solution. Let $\xi' = 2/(1-\eta)$.
   Consider $t_0>0$ and integer $m_0>0$ such that 
  $\{2 \kappa_+(m_0) b /m_0\}^{1/2} + \|\bX^\top\bep/n\|_\infty < \inf_{0<s<t_0}\drho(s;\lam)$, 
 where
 $b=\xi' \max\{\nu,\Delta( a_1'\lam^*_1,|S|;\lam \big)\}$ with 
 $a_1':= (1+\eta)/\RIF_1(\xi',S)$ and $\lam^*_1:=\sup_{t\ge 0}|\drho(t;\lam)|$. Then, 
\[
\#\{j\not\in S: \tbeta_j\neq 0\}  < \mtil := m_0 + \lfloor b/\rho(t_0;\lam)\rfloor .
\]
\end{theorem}

\begin{remark} If $\rho(t;\lam)$ is concave in $t$, then $\lam^*_1=\drho(0+;\lam)$ and 
$\inf_{0<s<t_0}\drho(s;\lam)$ can be replaced by $\drho(t_0;\lam)$ for choosing $(t_0,m_0)$. 
Theorem \ref{thm:local-global} applies to the $\ell_1$, capped-$\ell_1$, MCP and SCAD penalties with 
$\lam=\lam^*\asymp \lam^*_1$, but not to the bridge penalty for which $\lam^*_1=\infty$. 
\end{remark}

Theorem \ref{thm:local-global} 
shows that the $\ell_1$ solution $\hbbeta^{(\ell_1)}$ is $\{\nu,\bbeta\}$ approximately 
global optimal with $\nu=O(|S|(\lam^*)^2)$ in (\ref{eq:approx-global}), 
and that a local solution $\tbbeta$ which is also approximate global 
optimal is a sparse local solution. 
Thus, with $b=O((\lam^*)^2 |S|)$ and $\rho(t_0;\lam)\asymp (\lam^*)^2\asymp (\lam^*_1)^2$, 
the local solution $\tbbeta$ obtained with gradient descent from $\hbbeta^{(\ell_1)}$ 
is sparse with $\#\{j\not\in S: \tbeta_j\neq 0\} = O(|S|)$. 
Here we assume that a line-search is performed in the gradient descent procedure
so that the objective function always decreases (and thus each step leads to an $\{\nu,\bbeta\}$ approximate
global optimal solution). 
Now Remark~\ref{uniqueness} can be applied to this sparse local solution, providing  
suitable conditions for this solution to be identical to the global optimal solution. 
If $\min_{j\in S}|\beta_j|>C\lam_{univ}$ for a sufficiently large $C$, 
Corollary~\ref{cor:model_sel} (or Theorems~\ref{thm:approx-local} plus Theorem~\ref{thm:obeta}) 
can be applied to identify this local solution as the oracle LSE (or penalized 
LSE) and the global solution. 

It is worth pointing out results of this paper concerning the global solution
can be applied under the null consistency condition.
For a general penalty function, this requires the condition
(\ref{prop-2-1}) to hold. Although this is an $\ell_2$ condition, it isn't needed for either $\ell_1$ or $\ell_0$
penalty as pointed out in Remark~\ref{remark:null-consist}. In fact, this condition is also not needed
if we consider local solution obtained with more specific numerical procedures such as
\cite{Zhang10-mc+,ZhangTong11} that lead to specific sparse local solutions with oracle properties.
Nevertheless, it is useful to observe that if the extra condition (\ref{prop-2-1}) holds, then such a local
solution is also the unique global solution, and it can be obtained via other numerical procedures.

\section{Technical Proofs} 

We first prove the following two lemmas, which will be useful in the analysis.
\begin{lemma}\label{lem:lam}
  If $\hbbeta$ is the global solution of (\ref{eq:hbbeta}), then
  $\|\bX^\top(\by-\bX\hbbeta)/n\|_\infty \le \lam^*$. In particular, 
  $\|\bX^\top\bep/n\|_\infty\le\eta\lam^*$ under the $\eta$ null consistency condition (\ref{eq:null}). 
\end{lemma}
\begin{proof}
The optimality of $\hbbeta$ implies 
\[
\|\by-\bX\hbbeta\|^2_2/(2n) + \rho(\hbeta_j;\lam)
\le \|\by-\bX\hbbeta-\bx_jt\|^2_2/(2n) + \rho(\hbeta_j+t;\lam)
\]
for all real $t$. 
Since $\rho(t;\lam)$ is subadditive in $t$, 
\bes
t\bx_j^\top(\by-\bX\hbbeta)/n \le t^2\|\bx_j\|_2^2/(2n)+\rho(\hbeta_j+t;\lam)-\rho(\hbeta_j;\lam)
\le t^2/2+\rho(t;\lam). 
\ees
Since $t$ is arbitrary, we obtain the desired bound via the definition of $\lam^*$ in (\ref{eq:t-rho}).
\end{proof}

\begin{lemma}\label{lem:restricted}
Assume the null consistency condition (\ref{eq:null}) with $\eta\in(0,1)$.
Suppose $\hbbeta\in \R^p$ satisfy
\[
\|\by - \bX\hbbeta\|_2^2/(2n) + \|\rho(\hbbeta;\lam)\|_1
\leq \|\by - \bX\bbeta\|_2^2/(2n) + \|\rho(\bbeta;\lam)\|_1 + \nu 
\]
with a certain $\nu>0$. Let $\bDelta=\hbbeta-\bbeta$, $\xi=(1+\eta)/(1-\eta)$, and $S=\supp(\bbeta)$. Then, 
\[
\|\bX\bDelta\|_2^2/(2n) + \|\rho(\bDelta_{S^c};\lam)\|_1\le \xi\|\rho(\bDelta_S;\lam)\|_1 
 + \nu/(1-\eta). 
\]
\end{lemma}
\begin{proof}
From the condition of the lemma, we have
\bes
0 &\le& \nu + \|\by - \bX\bbeta\|_2^2/(2n) + \|\rho(\bbeta;\lam)\|_1 
- \|\by - \bX\hbbeta\|_2^2/(2n) - \|\rho(\hbbeta;\lam)\|_1
\cr &=& \nu - \|\bX\bDelta\|_2^2/(2n) + \bep^\top\bX\bDelta/n
+ \|\rho(\bbeta;\lam)\|_1 - \|\rho(\bbeta+\bDelta;\lam)\|_1. 
\ees
By (\ref{eq:null}), 
$\|\bep/\eta\|_2^2/(2n)\le \|\bep/\eta - {t} \bX\bDelta\|_2^2/(2n) + \|\rho({t}\bDelta;\lam)\|_1$ 
for all ${t}>0$, which can be written as 
\[
\bep^\top\bX\bDelta/n \le \eta {t} \|\bX\bDelta\|_2^2/(2n)+ (\eta/{t})\|\rho({t} \bDelta;\lam)\|_1 .
\]
The above two displayed inequalities yield
\begin{equation}
(1-\eta{t}) \|\bX\bDelta\|_2^2/(2n) - \nu \leq (\eta/{t})\|\rho({t}\bDelta;\lam)\|_1 
+ \|\rho(\bbeta;\lam)\|_1 - \|\rho(\bbeta+\bDelta;\lam)\|_1 . \label{eq:lem-restricted-1}
\end{equation}
Now let ${t}=1$. It follows from (\ref{eq:lem-restricted-1}), $\bbeta_{S^c}=0$, 
and then the sub-additivity of $\rho(t;\lam)$ that
\bes
(1- \eta) \|\bX\bDelta\|_2^2/(2n) -\nu
&\leq& \eta\|\rho(\bDelta;\lam)\|_1
+ \|\rho(\bbeta_S;\lam)\|_1 - \|\rho(\bbeta_S+\bDelta_S;\lam)\|_1 - \|\rho(\bDelta_{S^c};\lam)\|_1
\cr &\leq& (\eta+1)\|\rho(\bDelta_S;\lam)\|_1
+(\eta-1) \|\rho(\bDelta_{S^c};\lam)\|_1 . \qquad\qquad\qedhere
\ees
\end{proof}

\subsection{Proof of Proposition~\ref{prop:concave-Delta}}
Let $t>0$. By (\ref{eq:t-rho}), $\rho(t;\lam)\ge t(\lam^*  - t/2)\ge t\lam^*/2$ for $t\le\lam^*$. 
For $t>\lam^*$, $\rho(t;\lam)\ge\rho(\lam^*;\lam)\ge (\lam^*)^2/2$. This gives the 
lower bound of $\rho(t;\lam)$. 
Let $t_0$ be the minimizer in (\ref{eq:t-rho}) in the sense of 
$x/2+\rho(x;\lam)/x\to \lam^*$ as $x\to t_0$ (when $t_0$ is a discontinuity of $\rho(\cdot;\lam)$) 
or $x=t_0$. Let $x>0$ and $q=\lfloor t/x \rfloor$.
Since $\rho(t;\lam)$ is nondecreasing and subadditive in $t>0$, we have 
\begin{align*}
\rho(t;\lam)\le 
\rho(qx;\lam) + \rho(t-qx;\lam)
\leq (q+1) \rho(x;\lam) \leq (t+x)\rho(x;\lam)/x .
\end{align*}
It follows that (let $x \to t_0$)
$\rho(t;\lam)\le (t+t_0)(\lam^*-t_0/2)\le \max_{t' \geq 0} (t+t') (\lam^*-t'/2) = \rho^*(t;\lam)$. 
The bound for $\Delta(a,k;\lam)$ follows similarly from (let $x \to t_0$)
\[
\|\rho(\bb;\lam)\|_1\le\sum_{j: b_j \neq 0} (|b_j|+x)\rho(x;\lam)/x\le k(a+x)\rho(x;\lam)/x 
\leq k(a+t_0)(\lam^*-t_0/2)
\leq k \rho^*(a;\lam).
\]
The fact that $\rho^*(a;\lam) \leq \max(a;2\lam^*)\lam^*$ can be verified by simple algebra.
$\qed$ 

\subsection{Proof of Proposition~\ref{prop:invertibility}}
Let $f(t)=t/\rho(t;\lam)$ and $A$ be the index set of the 
$|S|$ largest $|u_j|$. Since $\rho(t;\lam)$ is nondecreasing in $|t|$, 
$\|\rho(\bu_{S^c};\lam)\|_1 < \xi \|\rho(\bu_S;\lam)\|_1$ implies 
$\|\rho(\bu_{A^c};\lam)\|_1 < \xi \|\rho(\bu_A;\lam)\|_1$. 
Since $f(t)$ is nondecreasing in $t$, 
\[
\|\bu_{A^c}\|_1\le \|\rho(\bu_{A^c};\lam)\|_1f(\|\bu_{A^c}\|_\infty)
\le \xi\|\rho(\bu_A;\lam)\|_1f(\|\bu_{A^c}\|_\infty) \le \xi \|\bu_{A}\|_1 . 
\]
This implies (\ref{RIF-CIF}). 
In the above derivation, 
the first inequality follows from the definition of $f(t)$ and $\|\bu_{A^c}\|_\infty \geq |\bu_j|$ for all $j \in A^c$;
the second inequality is due to the condition $\|\rho(\bu_{A^c};\lam)\|_1 < \xi \|\rho(\bu_A;\lam)\|_1$; 
the third inequality follows from the definition of $f(t)$ and the condition $\|\bu_{A^c}\|_\infty \leq |\bu_j|$  for all $j \in A$.
$\qed$

\subsection{Proof of Proposition \ref{prop-2}} 
Since the left-hand side of (\ref{eq:null}) is increasing in $\rho(t;\lam)$, 
we assume without loss of generality that 
\bes
\rho(t;\lam)= \min\big(\lam^2/2,\lam |t|\big),\ 
\lam = (1+\zeta_0)(\sigma/\eta)\lam_0,\ \lam_0 = \big(1+\sqrt{2\ln(2p/\delta)}\big)/\sqrt{n}.  
\ees
%Assume WOLG that $\rho(t;\lam)= \min\big(\lam^2/2,\lam |t|\big)$.
Since $\|\bX^\top \bep/n\|_\infty\le \max_{|A|=1}\|\bP_A\bep\|_2/\sqrt{n}$ 
and $\|\bP_A\bep\|_2 \leq \|\bep\|_2$, Assumption~\ref{assump:sub-Gaussian} implies that 
\begin{equation}
\|\bX^\top \bep/n\|_\infty \leq \sigma\lam_0,\ 
\|\bP_A\bep\|_2 \leq \sigma\min\Big[\sqrt{|A|n}\lam_0,\sqrt{2n}\Big], 
\ \forall\ A\subseteq\{1,\ldots,p\}, 
\label{eq:proof-prop-2-lam}
\end{equation}
with at least probability 
$1 - \exp(-n(1-1/\sqrt{2})^2) - \sum_{k=1}^n {p\choose k}\big(\delta/(2p)\big)^k \ge 2-\delta_n-e^{\delta/2}$. 

Let $A=\{j: |b_j|>\lam/2\}$ and $k=|A|$. It suffices to consider the case where $A$ and $\bb$ satisfy
\bes
\bX_A\bb_A = \bP_A(\bep/\eta-\bX_{A^c}\bb_{A^c}),\ \rank(\bP_A)=|A|=k
\le \frac{\|\bep/\eta\|_2^2/(2n)}{\lam^2/2} 
\le \frac{2}{(1+\zeta_0)^2\lam_0^2}, 
\ees
since these conditions hold for the global minimum for (\ref{eq:hbbeta}) 
with $\by=\bep/\eta$ and the capped-$\ell_1$ penalty. 
Under these conditions, we have 
$\bX\bb = \bP_A\bep/\eta + \bP_A^\perp\bX_{A^c}\bb_{A^c}$ and 
\bel{pf-prop-2-1}
&& \|\bep/\eta\|_2^2/(2n) - \|\bep/\eta-\bX\bb\|_2^2/(2n)-\|\rho(\bb;\lam)\|_1
\cr &=& (\bX\bb)^\top(\bep/\eta)/n - \|\bX\bb\|_2^2/(2n)-\|\rho(\bb;\lam)\|_1
\cr &=& \|\bP_A(\bep/\eta)\|_2^2/(2n) + (\bP_A^\perp\bX_{A^c}\bb_{A^c})^\top(\bep/\eta)/n 
- \|\bP_A^\perp\bX_{A^c}\bb_{A^c}\|_2^2/(2n)-\|\rho(\bb;\lam)\|_1
\cr &\le & \|\bP_A(\bep/\eta)\|_2^2/(2n) + (\bX_{A^c}\bb_{A^c})^\top(\bep/\eta)/n 
- (\bX_{A^c}\bb_{A^c})^\top\bP_A(\bep/\eta)/n -\|\rho(\bb;\lam)\|_1
\cr &\le& \lam_0^2k(\sigma/\eta)^2/2 + \lam_0(\sigma/\eta)\|\bb_{A^c}\|_1
- (\bX_{A^c}\bb_{A^c})^\top\bP_A(\bep/\eta)/n-\|\rho(\bb;\lam)\|_1
\cr &<& \Big(\frac{1}{1+\zeta_0}-1\Big)\|\rho(\bb;\lam)\|_1
- (\bX_{A^c}\bb_{A^c})^\top\bP_A(\bep/\eta)/n. 
\eel
In the above derivation, the second inequality uses (\ref{eq:proof-prop-2-lam}) 
and the third uses the fact that 
$\|\rho(\bb;\lam)\|_1=\lam^2k/2+\lam\|\bb_{A^c}\|_1$ by the definition of $A$ 
and $\lam = (1+\zeta_0)(\sigma/\eta)\lam_0$. 

It follows from the shifting inequality in \cite{CaiWX10,YeZ10} that  
\begin{align*}
\Big|\bep^\top\bP_A\bX_{A^c}\bb_{A^c}\Big|
\le& \max_{B\cap A=\emptyset, |B|=k} 
\|\bX_B^\top\bP_A\bep\|_2\Big(\|\bb_{A^c}\|_\infty k^{1/2} + \|\bb_{A^c}\|_1/k^{1/2}\Big) \\
\le& \max_{B\cap A=\emptyset, |B|=k}
\|\bX_B^\top\bP_A\bep\|_2(\lam k/2 + \|\bb_{A^c}\|_1)/\sqrt{k}. 
\end{align*}
In the above derivation, the first inequality uses the shifting inequality
and the second uses the fact that $\|\bb_{A^c}\|_\infty\le\lam/2$ 
due to the definition of $A$. 
It follows from (\ref{prop-2-1}) and $\|\bP_A\bep\|_2\le \sigma\lam_0\sqrt{nk}$ 
of (\ref{eq:proof-prop-2-lam}) that for all $|A|=|B|=k\le 2/\{(1+\zeta_0)^2\lam_0^2\}$ 
with $B\cap A=\emptyset$, 
\bes
\|\bX_B^\top\bP_A\bep\|_2
\le\lam_{\max}^{1/2}(\bX_B^\top\bP_A\bX_B) \|\bP_A\bep\|_2
\le (\sigma\lam_0\sqrt{nk})(\zeta_0\sqrt{n})=\sigma\zeta_0\lam_0n\sqrt{k}. 
\ees
Thus, by combining the above two displayed inequalities, we find 
\bes
\Big|\bep^\top\bP_A\bX_{A^c}\bb_{A^c}/(\eta n)\Big|
\le \frac{\sigma\zeta_0\lam_0n\sqrt{k}}{\eta n}\Big(\frac{\lam k/2 + \|\bb_{A^c}\|_1}{\sqrt{k}}\Big)
= (1+\zeta_0)^{-1}\zeta_0\|\rho(\bb;\lam)\|_1. 
\ees
due to $\lam=(1+\zeta_0)(\sigma/\eta)\lam_0$ and 
$\|\rho(\bb;\lam)\|_1=\lam^2 k/2 + \lam\|\bb_{A^c}\|_1$.
This and (\ref{pf-prop-2-1}) yield the null consistency condition (\ref{eq:null}). 

 %\bes
%\bP_A &=& \bX_A(\bX_A^\top\bX_A)^{-1}\bX_A^\top 
%\cr &=& \bN_{n\times p}\bU_A\bD_A\bW_A^\top
%(\bW_A\bD_A\bU_A^\top\bN_{n\times p}^\top\bN_{n\times p}\bU_A\bD_A\bW_A^\top)^{-1}
%\bW_A\bD_A\bU_A^\top\bN_{n\times p}^\top
%\cr &=& \bN_{n\times p}\bU_A
%(\bU_A^\top\bN_{n\times p}^\top\bN_{n\times p}\bU_A)^{-1}\bU_A^\top\bN_{n\times p}^\top
%\ees
It remains to prove that (\ref{prop-2-1}) is an $\ell_2$-regularity condition on $\bX$. 
Suppose that the rows of $\bX$ are iid from $N(0,\bSigma)$.  
Let $\bN_{k,m}$ denote a $k\times m$ matrix with iid $N(0,1)$ entries. 
We may write $\bX_B = \bN_{n,p}(\bSigma^{1/2})_{p\times B}$. 
Let $\bU\bU^T$ and $\bV\bD\bW^T$ be the SVDs of $\bP_A$ and $(\bSigma^{1/2})_{p\times B}$ 
respectively. For fixed $\{A,B\}$, the entries of the $k\times k$ matrix 
$\bU^T\bN_{n,p}\bV$ are uncorrelated $N(0,1)$ variables, so that we can write 
$\bP_A\bX_B=\bU\bN_{k,k}\bV^T(\bSigma^{1/2})_{p\times B}$. Thus, 
by Theorem II.13 of \cite{DavidsonS01}, 
\bes
P\big\{\lam_{\max}^{1/2}(\bX_B^\top\bP_A\bX_B) > (2k^{1/2}+t) \lam_{\max}^{1/2}(\bSigma)\big\} 
&\le& P\big\{\lam_{\max}^{1/2}(\bN_{k,k}^\top\bN_{k,k}) > 2k^{1/2}+t \big\} 
\cr &\le& \Phi(-t) \leq e^{-t^2/2}/2,\ t>0, 
\ees
where $\Phi(t)$ is the $N(0,1)$ distribution function. Since there are no more than 
${p\choose k,k,p-2k}$ choices of $\bP_A$ with rank $k$ and $|B|=k$ 
with $A\cap B=\emptyset$, 
\bel{pf-prop-2-2}
\max_{B\cap A=\emptyset, |B|=|A|=k}
\lam_{\max}^{1/2}(\bX_B^\top\bP_A\bX_B) 
\le \lam_{\max}^{1/2}(\bSigma)\Big(2k^{1/2}+\sqrt{8k \ln(2p/\delta)}\Big),\ \forall\ 1\le k\le n, 
\eel
with probability no smaller than 
\[
1- \frac{1}{2}\sum_{k=1}^n \Big(\frac{\delta}{2p}\Big)^{4k} {p\choose k,k,p-2k}
\ge 1- \frac{1}{2}\sum_{k=1}^n \frac{(\delta^2/(4p))^{2k}}{(k!)^2}
\ge 1 - \delta^4/(16p^2). 
\]
In the event (\ref{pf-prop-2-2}), we have that for all $|A|=|B|=k$ and 
$k(1+\zeta_0)^2(1+\sqrt{2\ln(2p/\delta)})^2\le 2n$, 
\bes
\lam_{\max}^{1/2}(\bX_B^\top\bP_A\bX_B/n)
\le \frac{\lam_{\max}^{1/2}(\bSigma)\big(2k^{1/2}+\sqrt{8k \ln(2p/\delta)}\big)}
{\{k^{1/2}(1+\zeta_0)(1+\sqrt{2\ln(2p/\delta)})\}/\sqrt{2}}
= \frac{\sqrt{8}\lam_{\max}^{1/2}(\bSigma)}{1+\zeta_0}. 
\ees
This proves the desired result. $\qed$

\subsection{Proof of Theorem \ref{thm:param-est}}
Let $\bDelta=\hbbeta-\bbeta$. 
Lemma~\ref{lem:restricted} (with $\nu=0$) implies that
\bel{pf-th-1-0}
\|\bX\bDelta\|_2^2/(2n) + \|\rho(\bDelta_{S^c};\lam)\|_1\le \xi\|\rho(\bDelta_S;\lam)\|_1. 
\eel
Thus, (\ref{eq:RIF}) gives 
\bel{pf-th-1-1}
\|\bDelta\|_q \le \|\bX^\top\bX\bDelta\|_\infty|S|^{1/q}/\{n\RIF_q(\xi,S)\}. 
\eel
It follows from Lemma~\ref{lem:lam} that $\|\bX^\top(\by - \bX\hbbeta)/n\|_\infty \le \lam^*$ 
and $\|\bX^\top\bep/n\|_\infty \le \eta\lam^*$. Thus, we have 
$\|\bX^\top\bX\bDelta/n\|_\infty = \|\bX^\top(\by - \bX\hbbeta - \bep)/n\|_\infty\le (1+\eta)\lam^*$. 
This and (\ref{pf-th-1-1}) yield (\ref{th-1-1}). 

Now by combining the definition of $\Delta(a,|S|;\lam)$ and
$\|\bDelta\|_1 \leq (1+\eta) \lam^* |S|/\RIF_1(\xi,S)$, which follows from (\ref{th-1-1}),
we obtain an estimate of $\|\rho(\bDelta_S;\lam)\|_1$ in (\ref{pf-th-1-0}), which leads to 
the first inequality in (\ref{th-1-2}). The second inequality in (\ref{th-1-2}) then follows from 
Proposition \ref{prop:concave-Delta} and Remark \ref{remark:Delta}.
$\qed$

\subsection{Proof of Theorem \ref{thm:sparsity}}
Let $\Shat_1=\{j \in \Shat \setminus S: |\hbeta_j|\geq t_0\}$ and 
$\Shat_2=\{j \in \Shat \setminus S: |\hbeta_j| < t_0\}$. 
As in the proof of (\ref{th-1-2}), 
it follows from the $\ell_1$ error bound (\ref{th-1-1}) and the definition of $\Delta(a_1,|S|;\lam)$ 
in (\ref{eq:Delta}) that $\|\rho(\bDelta_S;\lam)\|_1\le \Delta(a_1,|S|;\lam)$ with the given $a_1$. 
Thus, 
\bel{pf-th-2-1}
| \Shat_1| \le \|\rho(\bDelta_{S^c};\lam)\|_1 /\rho(t_0;\lam)\le\xi\|\rho(\bDelta_S;\lam)\|_1 /\rho(t_0;\lam)
\leq \xi \Delta(a_1,|S|;\lam) /\rho(t_0;\lam) .
\eel
Let $\lam_2>\sqrt{2\xi \kappa_+(m_0) \Delta(a_1\lam^*,|S|;\lam)/m_0}$ satisfying 
$\lam_2+\|\bX^\top\bep/n\|_\infty\le  \inf_{0<s<t_0}\drho(s;\lam)$. 
The first order optimality condition implies that for all $j \in \Shat$,
$\bx_j^\top(\by-\bX\hbbeta)/n=\drho(t;\lam)\big|_{t=\hbeta_j}$.
For $j\in \Shat_2$, $|\hbeta_j| \in (0,t_0)$, so that 
$|\bx_j^\top(\by-\bX\hbbeta)/n|\ge (\lam_2 + \|\bX^\top \bep/n\|_\infty)$ by (\ref{th-2-1}).
Thus, for any set $A\subset \Shat_2$ with $|A|\le m_0$, 
\[
(\lam_2 + \|\bX^\top \bep/n\|_\infty) |A| 
\le  \|\bX_A^\top(\by-\bX\hbbeta)/n\|_1
\le \|\bX_A^\top \bep/n\|_\infty |A| +|A|^{1/2}\|\bX_A/\sqrt{n}\|_2\|\bX\bDelta\|_2/\sqrt{n}. 
\]
Since $\|\bX_A/\sqrt{n}\|_2^2\le \kappa_+(m_0)$, 
$\lam_2|A| \le 
|A|^{1/2}\sqrt{\kappa_+(m_0)\|\bX\bDelta\|_2^2/n}$. 
It follows from Theorem \ref{thm:param-est} that
$|A| \le \kappa_+(m_0)\|\bX\bDelta\|_2^2/(n\lam_2^2)
\leq 2 \xi \kappa_+(m_0) \Delta(a_1,|S|;\lam)/\lam_2^2 < m_0$.
Thus, $\max_{A\subset \Shat_2, |A|\le m_0}|A| < m_0$, which implies that
$|\Shat_2|<m_0$.  Combine this estimate with (\ref{pf-th-2-1}), we obtain the desired bound.
$\qed$ 

\subsection{Proof of Theorem~\ref{thm:global-L0}}
It follows from the assumption of the theorem that for all $\bb \in \R^p$, 
$$ 
\|\bX\bb - \bep/\eta\|_2^2 + \lam^2 n \|\bb\|_0 - \|\bep/\eta\|_2^2 
= \|\bX \bb\|_2^2 + (2/\eta) \bep^\top X \bb + \lam^2 n \|\bb\|_0 
$$
is bounded from below by 
$\|\bX\bb\|_2^2 - 2\lam \sqrt{n \|\bb\|_0} \|\bX \bb\|_2
+ \lam^2 n \|\bb\|_0  = (\|\bX\bb\|_2^2 - \lam \sqrt{n \|\bb\|_0})^2 \geq 0$. 
This implies the null-consistency condition. 
Moreover, (\ref{eq:lem-restricted-1}) with ${t}=1/\eta$ and $\nu=0$ implies that
\[
 \|\hbbeta^{(\ell_0)}\|_0-\|\bbeta\|_0 \leq  \eta^2 \|\hbbeta^{(\ell_0)}-\bbeta\|_0  \leq 
\eta^2 \|\hbbeta^{(\ell_0)}\|_0 + \eta^2 \|\bbeta\|_0  ,
\]
which leads to the first bound of the theorem.
The second bound is a direct consequence of Theorem~\ref{thm:param-est}, 
since $\Delta(\xi, |S|;\lam) = \lam^2 |S|/2$ by (\ref{eq:Delta}). $\qed$

\subsection{Proof of Theorem~\ref{thm:global-L0-2}}
For simplicity, let $\hbbeta=\hbbeta^{(\ell_0)}$, $\Shat=\supp(\hbbeta)$,
and $S=\supp(\bbeta)$.
We know that $\|\hbbeta\|_0 \leq (1+\eta^2)/(1-\eta^2) \|\bbeta\|_0$ and thus
$\|\hbbeta-\obeta\|_0 \leq s$. 
Similar to the proof of Theorem~\ref{thm:global-L0}, we have 
\begin{align*}
0 \geq & \|\bX(\hbbeta-\obeta)\|_2^2 
+ 2 (\bX\obeta - \by)^\top \bX (\hbbeta-\obeta) 
+ \lam^2 n [\|\hbbeta\|_0 - \|\obeta\|_0 ] \\
\geq& \kappa_-(s) n \|\hbbeta-\obeta\|_2^2 - 
\sqrt{2\kappa_-(s)} \lam n \|(\hbbeta-\obeta)_{\Shat-S}\|_1
+ \lam^2 n [\|\hbbeta\|_0 - \|\obeta\|_0 ] \\
\geq& \kappa_-(s) n\|(\hbbeta-\obeta)_{S}\|_2^2 
+ \kappa_-(s) n\|(\hbbeta-\obeta)_{\Shat-S}\|_2^2 
- 2\sqrt{0.5 \lam^2 n |\Shat-S|} \sqrt{\kappa_-(s)n} \|(\hbbeta-\obeta)_{\Shat-S}\|_2 \\
& \qquad + \lam^2 n [\|\hbbeta\|_0 - \|\obeta\|_0 ] \\
\geq& \kappa_-(s) n \|(\obeta)_{S-\Shat}\|_2^2 
- 0.5 \lam^2 n |\Shat-S| 
+ \lam^2 n [\|\hbbeta\|_0 - \|\obeta\|_0 ] \\
\geq& 2 \lam^2 n (|S-\Shat| - \delta^{o})
- 0.5 \lam^2 n |\Shat-S| 
+ \lam^2 n [\|\hbbeta\|_0 - \|\obeta\|_0 ] \\
\geq& \lam^2 n  (|S-\Shat| + 0.5 |\Shat-S| - 2\delta^{o}) .
\end{align*}
The first inequality uses the same derivation of a similar result in the proof of Theorem~\ref{thm:global-L0}.
The second inequality uses the assumption of the theorem,
$(\bP_S \bep - \bep)^\top \bX=(\bX\obeta - \by)^\top \bX$,
and the fact that $(\bX\obeta - \by)^\top \bX_{S}=0$.
The forth inequality uses $b^2-2ab \geq -a^2$ and 
$\|(\hbbeta-\obeta)_{S}\|_2 \geq \|(\obeta)_{S-\Shat}\|_2$.
The fifth inequality uses
\begin{align*}
\kappa_-(s) n\|(\obeta)_{S-\Shat}\|_2^2 \geq&\kappa_-(s) n\sum_{j \in S-\Shat; |\hbeta_j|^2 \geq 2\lambda /\kappa_-(s)} (\obeta)_j^2 \\
\geq& 2 \lam^2 n \left|\{j \in S-\Shat; (\hbeta^{o}_j)^2 \geq 2\lam^2 /\kappa_-(s) \}\right| \geq
2 \lam^2 n  (|S-\Shat| - \delta^{o}) .
\end{align*}
The last inequality uses the derivation 
$\|\hbbeta\|_0 - \|\obeta\|_0  \geq |\Shat|-|S|= |\Shat-S|-|S-\Shat|$ and simple algebra.
This proves the first desired bound.
Similarly, we have
\begin{align*}
0\geq& \|\bX(\hbbeta-\obeta)\|_2^2 - 
\sqrt{2\kappa_-(s)} \lam n \|(\hbbeta-\obeta)_{\Shat-S}\|_1
+ \lam^2 n [\|\hbbeta\|_0 - \|\obeta\|_0 ] \\
\geq& 0.5 \|\bX(\hbbeta-\obeta)\|_2^2 + 0.5 \kappa_-(s)n \|(\hbbeta-\obeta)_{S}\|_2^2 
+ 0.5 \kappa_-(s)n \|(\hbbeta-\obeta)_{\Shat-S}\|_2^2 \\
& \quad - \sqrt{2\kappa_-(s) |\Shat-S|} \lam n \|(\hbbeta-\obeta)_{\Shat-S}\|_2
+ \lam^2 n [\|\hbbeta\|_0 - \|\obeta\|_0 ] \\
\geq& 0.5 \|\bX(\hbbeta-\obeta)\|_2^2 + 0.5 \kappa_-(s) n \|(\obeta)_{S-\Shat}\|_2^2 
-  \lam^2 n |\Shat-S| 
+ \lam^2 n [\|\hbbeta\|_0 - \|\obeta\|_0 ] \\
\geq& 0.5 \|\bX(\hbbeta-\obeta)\|_2^2 + \lam^2 n  (|S-\Shat| - \delta^{o})
- \lam^2 n |\Shat-S| 
+ \lam^2 n [\|\hbbeta\|_0 - \|\obeta\|_0 ] \\
\geq& 0.5 \|\bX(\hbbeta-\obeta)\|_2^2 - \lam^2 n \delta^{o} .
\end{align*}
The second inequality uses the definition of $\kappa_-(s)$.
The third inequality uses $0.5 b^2-\sqrt{2}ab \geq -a^2$ and 
$\|(\hbbeta-\obeta)_{S}\|_2 \geq \|(\obeta)_{S-\Shat}\|_2$.
The fourth inequality uses the previously derived inequality
$\kappa_-(s) n\|(\obeta)_{S-\Shat}\|_2^2 \geq 2 \lam^2 n  (|S-\Shat| - \delta^{o})$.
The last inequality uses the derivation 
$\|\hbbeta\|_0 - \|\obeta\|_0  \geq |\Shat|-|S|= |\Shat-S|-|S-\Shat|$ and simple algebra.
This leads to the second desired bound. $\qed$

\subsection{Proof of Theorem~\ref{thm:approx-local}}
Since $\tbbeta^{(j)}$ are approximate local solutions with excess $\nu^{(j)}$, (\ref{approx-local}) gives 
\bes
\|\bX^\top \bX\bDelta/n + \drho(\tbbeta^{(1)};\lam) - \drho(\tbbeta^{(2)};\lam)\|_2 
\leq (\nu^{(1)})^{1/2}+(\nu^{(2)})^{1/2} \le \sqrt{\nu}.
\ees
Let $E=\Stil^{(1)}\cup\Stil^{(2)}$. 
Since $|E|\le m+k$ and $\kappa < \kappa_-(m+k)$, it follows that 
\begin{align*}
\|\bX\bDelta\|_2^2/n
\leq & - \bDelta^\top (\drho(\tbbeta^{(1)};\lam) - \drho(\tbbeta^{(2)};\lam)) 
+ \sqrt{\nu}\|\bDelta\|_2 \\
\leq& \kappa \|\bDelta\|_2^2 + 
\left|(\bDelta^\top \theta(|\tbbeta^{(1)}|,\kappa) \right|+ \sqrt{\nu}\|\bDelta\|_2 \\
\leq& \kappa \|\bDelta\|_2^2 + 
\left(\|\theta(|\tbbeta^{(1)}_E|,\kappa)\|_2 + \sqrt{\nu}\right)  \|\bDelta\|_2 
%\\ \leq& \frac{\kappa\|\bX\bDelta\|_2^2}{n\kappa_-(m+k)} + 
%\left(\|\theta(|\tbbeta^{(1)}_E|,\kappa)\|_2 + \sqrt{\nu}\right)\frac{\|\bX\bDelta\|_2}{\sqrt{n\kappa_-(m+k)}}.
\end{align*}
Since $\|\bDelta\|_2^2\le \|\bX\bDelta\|_2^2/\{n\kappa_-(m+k)\}$, (\ref{th-5-1}) follows.  

Let $E_1:=\{j: |\tbeta_j^{(1)}-\tbeta_j^{(2)}| \geq \lam_0/\sqrt{\kappa_-(m+k)}\}$. 
We have $\lam_0^2|E_1|\le \kappa_-(m+k)\|\bDelta\|_2^2\le \|\bX\bDelta\|_2^2/n$. 
Since $j\in S\setminus\Stil^{(2)}$ implies $\tbeta_j^{(1)}-\tbeta_j^{(2)}=\tbeta_j^{(1)}$,  
(\ref{th-5-2}) follows. 

Let $E_2:= \Stil^{(2)}\setminus S$ and 
$\lam_0'=\drho(0+;\lam) - \|\bX_{S^c}^\top (\bX \tbbeta^{(1)}-\by )/n\|_\infty$. 
For $j\in E_2$, 
\bes
\lam_0' &\le& \drho(0+;\lam) + \sgn(\tbbeta^{(2)})\bx_j^\top (\bX \tbbeta^{(1)}-\by )/n 
\cr &\le& \big\{\drho(0+;\lam) - \sgn(\tbbeta^{(2)})\drho(\tbeta_j^{(2)};\lam)\big\}
+ |\bx_j^\top (\bX \tbbeta^{(2)}-\by)/n + \drho(\tbeta_j^{(2)};\lam)| 
+ |\bx_j^\top\bX\bDelta/n| .
\ees
Since $\theta(0+,\kappa)=0$ means $\drho(0+;\lam) - \sgn(t) \drho(t;\lam) =\drho(0+;\lam) - \drho(|t|;\lam)\le\kappa|t|$ for $t\neq 0$ 
and $\tbbeta^{(2)}_{E_2}=-\bDelta_{E_2}$, 
\bes
|E_2|\lam_0' &\le& \kappa\|\bDelta_{E_2}\|_1+
\|\bX_{E_2}^\top(\bX \tbbeta^{(2)}-\by)/n + \drho(\tbbeta^{(2)}_{E_2};\lam)\|_1
+ \|\bX_{E_2}^\top\bX\bDelta/n\|_1
\cr &\le& \sqrt{|E_2|}\Big\{\kappa\|\bDelta\|_2+
\sqrt{\tnu^{(2)}} + \|\bX_{E_2}^\top\bX\bDelta/n\|_2\Big\}
\ees
Since $\|\bDelta\|_2^2\le \|\bX\bDelta\|_2^2/\{n\kappa_-(m+k)\}$ and 
$\|\bX_{E_2}^\top\bX\bDelta/n\|_2^2\le \kappa_+(m)\|\bX\bDelta\|_2^2/n$, (\ref{th-5-3}) follows. 
$\qed$ 

\subsection{Proof of Theorem~\ref{thm:obeta}}
We note that $\bX^\top (\by - \bX \obeta)=\bX^\top (\bep - \bP_S \bep )=\bX^\top \bP_S^\perp \bep$.

(i) Since $\bx_j^\top (\bX \obeta -\by)/n + \drho(\hbeta^{o}_j;\lam)=\drho(\hbeta^{o}_j;\lam)$ for $j \in S$ and
$\bx_j^\top (\bX \obeta -\by)/n + \drho(\hbeta^{o}_j;\lam)=0$ for $j \notin S$, 
$\nu=\|\drho(\obeta_S;\lam)\|^2$. Let $\hbbeta^{(2)}$ be the global solution of (\ref{eq:hbbeta}). 
Under the additional conditions, 
$\hbbeta^{(2)}=\obeta$ by Theorems \ref{thm:sparsity} and (\ref{th-5-1}) with 
$\hbbeta^{(1)}=\obeta$.

(ii) Since $\min_{j\in S}|\hbeta^{o}_j|\ge \theta_1\lam^*$, the map 
$\bb_S \to \obeta_S - (\bX_S^\top\bX_S/n)^{-1}\drho(\bb_S;\lam)$
is continuous and closed in the rectangle 
$B = \{\bv: \|\bv_S-\obeta_S\|_\infty\le\theta_1\lam^*,\bv_{S^c}=0\}$. 
Thus, the Brouwer fixed point theorem implies a $\hbbeta\in B$ 
satisfying $\sgn(\hbbeta)=\sgn(\bbeta)$ and
\bes
\bX_S^\top(\by - \bX\hbbeta) = (\bX_S^\top\bX_S/n)(\obeta-\hbbeta)_S=\drho(\hbbeta_S;\lam). 
\ees
Since $\by-\bX\hbbeta = \by-\bX\obeta-\bX_S(\hbbeta-\obeta)_S
=\by-\bX\obeta-\bX_S(\bX_S^\top\bX_S/n)^{-1}\drho(\hbbeta_S;\lam)$, 
\bes
\|\bX^\top_{S^c}(\by-\bX\hbbeta)/n\|_\infty \le 
\|\bX^\top_{S^c}(\by-\bX\obeta)/n\|_\infty + \theta_2\lam^*\le\lam^*,
\ees
so that $\hbbeta$ is a local solution of (\ref{eq:hbbeta}). 
The proof of global optimality of $\tbbeta^{o}$ is the same as (i). $\qed$

\subsection{Proof of Theorem~\ref{thm:local-global}}
The proof is similar to that of Theorem~\ref{thm:sparsity}. As intermediate results, we will prove lemmas that are analogous 
to Lemma~\ref{lem:lam} and Theorem~\ref{thm:param-est}.
In the following, we assume that the conditions of the theorem hold.
We also let  $\bDelta=\tbbeta-\bbeta$.

\begin{lemma} \label{lem:lam-local} Let $\lam^*_1:=\sup_{t \geq 0} |\drho(t;\lam)|$.  
We have $\|\bX^\top(\by-\bX\tbbeta)/n\|_\infty \le \lam^*_1$. 
\end{lemma}

\begin{proof} 
  A local solution satisfies 
  $|\bx_j^\top (\bX \tbbeta - \by)/n| =|\drho(\tbeta_j;\lam)|\le \lam^*_1$ for all $j$. 
\end{proof}

\begin{lemma} \label{lem:param-est-local}
  We have $\|\bX\bDelta\|_2^2/(2n) + \|\rho(\bDelta_{S^c};\lam)\|_1  \le b$ with $\bDelta=\tbbeta-\bbeta$. 
\end{lemma}
\begin{proof} 
  We consider two situations: the first is 
  $\|\rho(\bDelta_S;\lam)\|_1 \leq \nu$, and the second is $\|\rho(\bDelta_S;\lam)\|_1 > \nu$.
  In the first situation, we obtain directly from Lemma~\ref{lem:restricted} that
  \[
  \|\bX\bDelta\|_2^2/(2n) + \|\rho(\bDelta_{S^c};\lam)\|_1  \le 2 \nu/(1-\eta) =\xi'\nu . 
  \] 
  In the second situation, we obtain from Lemma~\ref{lem:restricted} that
  $\|\bX\bDelta\|_2^2/(2n) + \|\rho(\bDelta_{S^c};\lam)\|_1\le \xi' \|\rho(\bDelta_S;\lam)\|_1$.
  Therefore (\ref{eq:RIF}) gives 
  $\|\bDelta\|_1 \le \|\bX^\top\bX\bDelta\|_\infty|S|/\{n\RIF_1(\xi',S)\}$.
 It follows from Lemma~\ref{lem:lam-local} that $\|\bX^\top(\by - \bX\hbbeta)/n\|_\infty \le \lam^*_1$. 
  Similarly, $\|\bX^\top(\bep/\eta)/n\|_\infty \le \lam^*=\lam$ due to (\ref{eq:null}). 
  Since $\lam=\lam^*\le \inf_t |\rho(t;\lam)/t| \leq \lam^*_1$, we have
  $\|\bX^\top\bX\bDelta/n\|_\infty = \|\bX^\top(\by - \bX\hbbeta - \bep)/n\|_\infty\le (1+\eta) \lam^*_1$.
 This implies that  $\|\bDelta\|_1 \le a'_1 \lam^*_1|S|$, where $a_1'=(1+\eta)/\RIF_1(\xi',S)$. 
  This can be combined with Lemma~\ref{lem:restricted} and the definition of $\Delta(a,|S|;\lam)$ to obtain 
  \[
  \|\bX\bDelta\|_2^2/(2n) + \|\rho(\bDelta_{S^c};\lam)\|_1  \le \xi' \Delta(a'_1 \lam^*_1,|S|;\lam) .
  \]
  Combine the two situations, we obtain the lemma.
\end{proof}

We are now ready to prove the theorem. 

(i) Let $\bDelta^{(\ell_1)}=\hbbeta^{(\ell_1)}-\bbeta$. 
Since $|\bep^\top\bX\bDelta^{(\ell_1)}/n|\le \|\bX\bDelta^{(\ell_1)}\|_2^2/(2n)
+\|\rho(\bDelta^{(\ell_1)};\lam)\|_1$ by (\ref{eq:null}), 
\bes
\nu &=& \|\bX\bDelta^{(\ell_1)}\|_2^2/(2n) - \bep^\top\bX\bDelta^{(\ell_1)}/n + \|\rho(\hbbeta^{(\ell_1)};\lam)\|_1 
- \|\rho(\bbeta;\lam)\|_1
\cr &\le& 2\Big\{\|\bX\bDelta^{(\ell_1)}\|_2^2/(2n) + \|\rho(\bDelta^{(\ell_1)};\lam)\|_1\Big\}
\cr &\le & 2\Big\{\xi a_1\lam^2 |S| + \Delta(a_1\lam |S|/m,m;\lam)\Big\}
\cr &\le & 2\Big\{\xi a_1\lam^2|S| + \lam^2\max(a_1|S|, 2m)\Big\} = O(\lam^2|S|). 
\ees

(ii) Let $\Shat_1=\{j \in \Shat \setminus S: |\hbeta_j|\geq t_0\}$, 
$\Shat_2=\{j \in \Shat \setminus S: |\hbeta_j| < t_0\}$, and 
$\lam_2>\sqrt{2\kappa_+(m_0)b/m_0}$ satisfying 
$\lam_2+\|\bX^\top\bep/n\|_\infty<\inf_{0<s<t_0}\drho(s;\lam)$. 
Just as in the proof of Theorem~\ref{thm:sparsity}, we have 
$| \Shat_1|\le \|\rho(\bDelta_{S^c};\lam)\|_1/\rho(t_0;\lam)$, 
and for any $A \subset \Shat_2$ with $|A| \leq m_0$, 
$|A| \le \kappa_+(m_0)\|\bX\bDelta\|_2^2/(n\lam_2^2)$. 
We apply Lemma~\ref{lem:param-est-local} to obtain
$|\Shat_1|\leq b/\rho(t_0; \lam)$ and $|A| \leq 2\kappa_+(m_0)b/\lam_2^2 < m_0$. 
Thus, $\max_{A\subset \Shat_2, |A|\le m_0}|A| < m_0$, which implies that $|\Shat_2|<m_0$. 
The theorem follows. $\qed$ 

\bibliographystyle{abbrv}
\bibliography{nonconvex}

\begin{thebibliography}{10}

\bibitem{Akaike73}
H.~Akaike.
\newblock Information theory and an extension of the maximum likelihood
  principle.
\newblock In {\em Second International Symposium on Information Theory}, pages
  267--281. Akademiai Kiad\'o, Budapest, 1973.

\bibitem{Antoniadis10}
A.~Antoniadis.
\newblock Comments on: $\ell_1$-penalization for mixture regression models.
\newblock {\em TEST}, 19:257--258, 2010.

\bibitem{BickelRT09}
P.~Bickel, Y.~Ritov, and A.~Tsybakov.
\newblock Simultaneous analysis of {L}asso and {D}antzig selector.
\newblock {\em Annals of Statistics}, 37:1705--1732, 2009.

\bibitem{BrehenyH11}
P.~Breheny and J.~Huang.
\newblock Coordinate descent algorithms for nonconvex penalized regression,
  with applications to biological feature selection.
\newblock {\em Annals of Applied Statistics}, 5:232--253, 2011.

\bibitem{BuhlmannGeer11}
P.~B{\"u}hlmann and S.~van~de Geer.
\newblock {\em Statistics for High-Dimensional Data: Methods, Theory and
  Applications}.
\newblock Springer, New York, 2011.

\bibitem{BuneaTW07}
F.~Bunea, A.~Tsybakov, and M.~Wegkamp.
\newblock Sparsity oracle inequalities for the lasso.
\newblock {\em Electronic Journal of Statistics}, 1:169--194, 2007.

\bibitem{CaiWX10}
T.~Cai, L.~Wang, and G.~Xu.
\newblock Shifting inequality and recovery of sparse signals.
\newblock {\em IEEE Transactions on Signal Processing}, 58:1300--1308, 2010.

\bibitem{CandesPlan09}
E.~Candes and Y.~Plan.
\newblock Near-ideal model selection by $\ell_1$ minimization.
\newblock {\em Annals of Statistics}, 37:2145--2177, 2009.

\bibitem{CandesTao07}
E.~Candes and T.~Tao.
\newblock The {D}antzig selector: statistical estimation when $p$ is much
  larger than $n$ (with discussion).
\newblock {\em Annals of Statistics}, 35:2313--2404, 2007.

\bibitem{CandesTao05}
E.~J. Candes and T.~Tao.
\newblock Decoding by linear programming.
\newblock {\em IEEE Trans. on Information Theory}, 51:4203--4215, 2005.

\bibitem{ChenDS01}
S.~S. Chen, D.~L. Donoho, and M.~A. Saunders.
\newblock Atomic decomposition by basis pursuit.
\newblock {\em SIAM Review}, 43:129--159, 2001.

\bibitem{DavidsonS01}
K.~Davidson and S.~Szarek.
\newblock Local operator theory, random matrices and {B}anach spaces.
\newblock In {\em Handbook on the Geometry of Banach Spaces}, volume~1. 2001.

\bibitem{EfronHJT04}
B.~Efron, T.~Hastie, I.~Johnstone, and R.~Tibshirani.
\newblock Least angle regression (with discussion).
\newblock {\em Annals of Statistics}, 32:407--499, 2004.

\bibitem{FanLi01}
J.~Fan and R.~Li.
\newblock Variable selection via nonconcave penalized likelihood and its oracle
  properties.
\newblock {\em Journal of the American Statistical Association}, 96:1348--1360,
  2001.

\bibitem{FanP04}
J.~Fan and H.~Peng.
\newblock On non-concave penalized likelihood with diverging number of
  parameters.
\newblock {\em Annals of Statistics}, 32:928--961, 2004.

\bibitem{FrankF93}
I.~Frank and J.~Friedman.
\newblock A statistical view of some chemometrics regression tools (with
  discussion).
\newblock {\em Technometrics}, 35:109--148, 1993.

\bibitem{GreenshteinR04}
E.~Greenshtein and Y.~Rotiv.
\newblock Persistence in high--dimensional linear predictor selection and the
  virtue of overparametrization.
\newblock {\em Bernoulli}, 10:971--988, 2004.

\bibitem{HuangMZ08}
J.~Huang, S.~Ma, and C.-H. Zhang.
\newblock Adaptive lasso for sparse high--dimensional regression models.
\newblock {\em Statistica Sinica}, 18:1603--1618, 2008.

\bibitem{HunterLi05}
D.~Hunter and R.~Li.
\newblock Variable selection using {MM} algorithms.
\newblock {\em Annals of Statistics}, 33:1617--1642, 2005.

\bibitem{KiChOh08}
Y.~Kim, H.~Choi, and H.-S. Oh.
\newblock Smoothly clipped absolute deviation on high dimensions.
\newblock {\em Journal of American Statistical Association}, 103:1665--1673,
  2008.

\bibitem{KnightF00}
K.~Knight and W.~Fu.
\newblock Asymptotics for lasso-type estimators.
\newblock {\em Annals of Statistics}, 28:1356--1378, 2000.

\bibitem{Koltchinskii09}
V.~Koltchinskii.
\newblock The dantzig selector and sparsity oracle inequalities.
\newblock {\em Bernoulli}, 15:799--828, 2009.

\bibitem{LiWoYe10}
J.~Liu, P.~Wonka, and J.~Ye.
\newblock Multi-stage {D}antzig selector.
\newblock In {\em NIPS 10}. 2010.

\bibitem{Mallows73}
C.~Mallows.
\newblock Some comments on cp.
\newblock {\em Technometrics}, 12:661--675, 1973.

\bibitem{MazumderFH11}
R.~Mazumder, J.~Friedman, and T.~Hastie.
\newblock Sparsenet : Coordinate descent with non-convex penalties.
\newblock {\em Journal of American Statistical Association}, page in press,
  2011.

\bibitem{MeinshausenB06}
N.~Meinshausen and P.~B{\"u}hlmann.
\newblock High dimensional graphs and variable selection with the {L}asso.
\newblock {\em Annals of Statistics}, 34:1436--1462, 2006.

\bibitem{MeinshausenY09}
N.~Meinshausen and B.~Yu.
\newblock Lasso--type recovery of sparse representations for high--dimensional
  data.
\newblock {\em Annals of Statistics}, 37:246--270, 2009.

\bibitem{OsbornePT00a}
M.~Osborne, B.~Presnell, and B.~Turlach.
\newblock A new approach to variable selection in least squares problems.
\newblock {\em IMA Journal of Numerical Analysis}, 20:389--404, 2000.

\bibitem{OsbornePT00b}
M.~Osborne, B.~Presnell, and B.~Turlach.
\newblock On the lasso and its dual.
\newblock {\em Journal of Computational and Graphical Statistics},
  9(2):319--337, 2000.

\bibitem{RaskuttiWY09}
G.~Raskutti, M.~J. Wainwright, and B.~Yu.
\newblock Minimax rates of estimation for high--dimensional linear regression
  over $\ell_q$--balls.
\newblock Technical report, University of California, Berkeley, 2009.

\bibitem{Schwarz78}
G.~Schwarz.
\newblock Estimating the dimension of a model.
\newblock {\em Annals of Statistics}, 6:461--464, 1978.

\bibitem{StadlerBGeer10}
N.~St{\"a}dler, P.~B{\"u}hlmann, and S.~van~de Geer.
\newblock $\ell_1$-penalization for mixture regression models (with
  discussion).
\newblock {\em Test}, 2:209--285, 2010.

\bibitem{SunZ10}
T.~Sun and C.-H. Zhang.
\newblock Comments on: $\ell_1$-penalization for mixture regression models.
\newblock {\em Test}, 2:270--275, 2010.

\bibitem{SunZ11}
T.~Sun and C.-H. Zhang.
\newblock Scaled sparse linear regression.
\newblock Technical Report arXiv:1104.4595, arXiv, 2011.

\bibitem{Tibshirani96}
R.~Tibshirani.
\newblock Regression shrinkage and selection via the lasso.
\newblock {\em Journal of the Royal Statistical Society: Series B (Statistical
  Methodology)}, 58:267--288, 1996.

\bibitem{Tropp06}
J.~A. Tropp.
\newblock Just relax: convex programming methods for identifying sparse signals
  in noise.
\newblock {\em IEEE Transactions on Information Theory}, 52:1030--1051, 2006.

\bibitem{vandeGeer07}
S.~van~de Geer.
\newblock The deterministic {L}asso.
\newblock Technical Report 140, ETH Zurich, Switzerland, 2007.

\bibitem{vandeGeer08}
S.~van~de Geer.
\newblock High--dimensional generalized linear models and the {L}asso.
\newblock {\em Annals of Statistics}, 36:614--645, 2008.

\bibitem{vandeGeerB09}
S.~van~de Geer and P.~B{\"u}hlmann.
\newblock On the conditions used to prove oracle results for the {L}asso.
\newblock {\em Electronic Journal of Statistics}, 3:1360--1392, 2009.

\bibitem{Wainwright09}
M.~J. Wainwright.
\newblock Sharp thresholds for noisy and high--dimensional recovery of sparsity
  using $\ell_1$--constrained quadratic programming ({L}asso).
\newblock {\em IEEE Transactions on Information Theory}, 55:2183--2202, 2009.

\bibitem{YeZ10}
F.~Ye and C.-H. Zhang.
\newblock Rate minimaxity of the lasso and dantzig selector for the $\ell_q$
  loss in $\ell_r$ balls.
\newblock {\em Journal of Machine Learning Research}, 11:3519--3540, 2010.

\bibitem{Zhang10-mc+}
C.-H. Zhang.
\newblock Nearly unbiased variable selection under minimax concave penalty.
\newblock {\em The Annals of Statistics}, 38:894--942, 2010.

\bibitem{ZhangHuang08}
C.-H. Zhang and J.~Huang.
\newblock The sparsity and bias of the {L}asso selection in high--dimensional
  linear regression.
\newblock {\em Annals of Statistics}, 36:1567--1594, 2008.

\bibitem{ZhangT09}
T.~Zhang.
\newblock Some sharp performance bounds for least squares regression with $l_1$
  regularization.
\newblock {\em Annals of Statistics}, 37:2109--2144, 2009.

\bibitem{zhang09-multistage}
T.~Zhang.
\newblock Analysis of multi-stage convex relaxation for sparse regularization.
\newblock {\em Journal of Machine Learning Research}, 11:1087--1107, 2010.

\bibitem{Zhang11-foba}
T.~Zhang.
\newblock Adaptive forward-backward greedy algorithm for learning sparse
  representations.
\newblock {\em IEEE Transactions on Information Theory}, 57:4689--4708, 2011.

\bibitem{ZhangTong11}
T.~Zhang.
\newblock Multi-stage convex relaxation for feature selection.
\newblock Technical Report arXiv:1106.0565, arXiv, 2011.

\bibitem{ZhaoY06}
P.~Zhao and B.~Yu.
\newblock On model selection consistency of {L}asso.
\newblock {\em Journal of Machine Learning Research}, 7:2541--2567, 2006.

\bibitem{Zou06}
H.~Zou.
\newblock The adaptive {L}asso and its oracle properties.
\newblock {\em Journal of the American Statistical Association},
  101:1418--1429, 2006.

\bibitem{ZouLi08}
H.~Zou and R.~Li.
\newblock One-step sparse estimates in nonconcave penalized likelihood models
  (with discussion).
\newblock {\em Annals of Statistics}, 36(4):1509--1533, 2008.

\end{thebibliography}

\end{document}